\documentclass{article}

\usepackage{PRIMEarxiv}

\usepackage[utf8]{inputenc} 
\usepackage[T1]{fontenc}    

\usepackage{algorithm}
\usepackage{amsthm}
\usepackage{amsmath}
\usepackage{algpseudocode}
\usepackage{multirow}
\usepackage{booktabs}
\usepackage{multicol}
\newtheorem{lemma}{Lemma}

\usepackage{graphicx}
\usepackage{subcaption}
\title{Leveraging Previous Steps: A Training-free Fast Solver for Flow Diffusion}

\author{
  Kaiyu Song, Hanjiang Lai \\
  Sun Yat-Sen University \\
  \texttt{\{songky7, laihanj3\}@mail2.sysu.edu.cn}
}
\begin{document}
\maketitle

\begin{abstract}
Flow diffusion models (FDMs) have recently shown potential in generation tasks due to the high generation quality. However, the current ordinary differential equation (ODE) solver for FDMs, e.g., the Euler solver, still suffers from slow generation since ODE solvers need many number function evaluations (NFE) to keep high-quality generation. In this paper, we propose a novel training-free flow-solver to reduce NFE while maintaining high-quality generation. The key insight for the flow-solver is to leverage the previous steps to reduce the NFE, where a cache is created to reuse these results from the previous steps. Specifically, the Taylor expansion is first used to approximate the ODE. To calculate the high-order derivatives of Taylor expansion, the flow-solver proposes to use the previous steps and a polynomial interpolation to approximate it, where the number of orders we could approximate equals the number of previous steps we cached. We also prove that the flow-solver has a more minor approximation error and faster generation speed. Experimental results on the CIFAR-10, CelebA-HQ, LSUN-Bedroom, LSUN-Church, ImageNet, and real text-to-image generation prove the efficiency of the flow-solver. Specifically, the flow-solver improves the FID-30K from 13.79 to 6.75, from 46.64 to 19.49 with $\text{NFE}=10$ on CIFAR-10 and LSUN-Church, respectively.

\end{abstract}    
\section{Introduction}
\label{sec:intro}


Flow diffusion models (FDMs)~\cite{flow_for_generative,rectified_flow}, which learn ordinary differential equation (ODE) to transport between noise and data distributions, are one of the generative models that could generate high-quality images. However, diffusion-based models always need many function evaluations (NFE)~\cite{sde}, which causes slow generation speed. Previous accelerating works~\cite{consistency_model,freeu} include training-based and training-free methods. Training-based methods rely on distillation~\cite{consistency_model,fast5}, which uses a student model with smaller time steps to distill the original diffusion models. Training-free methods~\cite{dpm_solver,dc_solver,unipc} aim to decrease the sampling steps by proposing a fast solver. This paper focuses on the training-free methods to accelerate FDMs~\cite{flow_for_generative,rectified_flow}.


On the one hand, since the diffusion models~\cite{sde,ddim,rectified_flow} can be mainly categorized as FDMs and diffusion probabilistic models (DPMs), several training-free accelerated methods~\cite{dc_solver,dpm_solver,unipc,deis,pseudo} have been proposed for DPMs. DPMs~\cite{sde,ddim} leverage a score function modeled by a neural network to formulate the probabilistic path. To reduce the NFE, DPM-solver~\cite{dpm_solver} first proposed an acceleration method by proposing a unique general solution for the ODE. UniPC~\cite{unipc} improved the DPM-solver, leveraging a unified predictor-corrector framework. DC-solver~\cite{dc_solver} further improved the UniPC by introducing a dynamic compensation strategy.



On the other hand, unlike the DPMs, FDMs are based on the conditional normalization flow~\cite{flow_for_generative}, which unifies a specific ODE to improve the training paradigm. Recently, the accelerating methods for FDMs have mainly focused on training-based methods. For example, the self-distillation~\cite{rectified_flow} has proposed to achieve one-step generation. The latest work is the bespoke Non-Stationary Solvers~\cite{non-station}, which proposed a light-aware training-based accelerating method. A few works have been proposed for training-free methods. Euler solver~\cite{rectified_flow} is a classic approach for ODE. Further, the Heun solver~\cite{diffusers} is proposed to improve the Euler solver to reduce the approximation errors. Both Euler and Heun solvers need many NFEs. For example, the Heun solver needs an additional function evaluation, which increases the NFE.

In this work, we propose the flow-solver, the novel, and training-free fast solver, to accelerate the sampling process of FDMs. The key insight is to reuse the previous steps to reduce the approximation errors. Since we can create a cache to collect the previous function evaluation results, thus we can also reduce the NFE. Concretely, the flow-solver first uses the previous steps to approximate the high-order derivatives of Taylor expansion of ODE, where the distances between the results of the previous steps and those of the current step are calculated. Then, a polynomial interpolation is used to solve the mismatch coefficients. Leveraging these allows us to make the previous steps to approximate the Taylor series of the velocity function in the next step. Our proposed flow-solver could more precisely approximate the continuous integral in the large time interval. The experimental results in CIFAR-10, CelebA-HQ, LSUN-Church, LSUN-Bedroom, ImageNet, and text-to-image conditional generation tasks show that our flow-solver makes a significant improvement for FDMs compared to the existing solvers, which prove the validity of our fast solver.

To sum up, our main contributions are:
\begin{itemize}
    \item We propose a novel flow-solver for FDMs, which can generate high-quality images compared to the existing solvers in fewer NFE.
    \item The flow-solver explores the connection between the previous steps and the high-order derivatives of Taylor expansion, which ensures the training-free acceleration for FDMs.
    \item We conduct extensive experiments on various generation tasks with three different FDMs.
    The experiments show that our flow-solver achieves the best performance.
\end{itemize}

\section{Related Work}
\label{sec:related_work}
\textbf{Accelerating methods for DPMs.} The reverse process of the DPMs could be represented in three ways: 1) \textit{Variance inference.}~\cite{ddim,nmdm,ddpm}, 2) \textit{Stochastic differential equation}. (SDE)~\cite{sde,gotta} and 3) \textit{ODE}~\cite{consistencytrajectorymodelslearning}. Hence, training-free acceleration for DPMs mainly focuses on how to leverage the property of each way. In the first way, based on the variance inference, the DDIM has been proposed~\cite{ddim} by changing the Markov chain in DPMs to a non-Markov structure with the deterministic process, thus accelerating the sampling process. 

For the second type, previous works~\cite{edm} focus on designing a better noise scheduler design for SDE  to improve the drift and diffusion coefficient factors. For example, EDM~\cite{edm,edm2} first proposed to scale the time interval from $[0,1]$ to $[80.0,0.002]$. AYS~\cite{align} proposed an optimized search strategy to design a better noise scheduler based on the SDE. Except for the improvement in the noise scheduler, Liu \textit{et al.}~\cite{pseudo} and DEIS~\cite{deis} proposed to re-formulate the SDE to be compatible with the block-box fast solver of SDE. 

For the third type, previous work on ODE focuses on exploring the connection hidden in the ODE with high-order derivatives~\cite{num_ode,genie}. Firstly, the DPM-solver proposed a unique general solution to enable high-order derivatives. Based on this, the DPM-solver also provided a solution to increase the solver order from $1$ to $p$, where $p$ is the highest order. Then, DPM-solver+=~\cite{dpm++} improved the DPM-solver by implementing the high-order solver from the noise to pixel space since there are fewer approximation errors for approximating high-order derivatives. 

DPM-solver offers a view that the high-order solver could accelerate the sampling process of DPMs. Thus, UniPC~\cite{unipc} introduced the predictor-corrector framework to further increase the order of the DPM-solver by introducing the previous results. Meanwhile, DC-solver~\cite{dc_solver} proposed a lightning-aware training strategy to improve the UniPC further.

Meanwhile, there are also works~\cite{consistencytrajectorymodelslearning,freeu,earlystop} that accelerate the sampling process based on the interesting findings. For example, Free-UNet~\cite{freeu} found some redundant steps that could be replaced by previous results, thus leveraging a cache to avoid additional calculating. ES-DDPM~\cite{earlystop} proposed to use the early-stop strategy to skip the redundant steps.

\textbf{Acceleration methods for FDMs.} FDMs unify to represent the probabilistic path as a specific ODE. This enables FDMs to achieve acceleration in a training-based way. For example, Rectified flow~\cite{rectified_flow} proposed that the Reflow operation could accelerate a straight probabilistic path. Then, Instaflow~\cite{instaflow} proposed to change the DPMs to FDMs first, then implementing Reflow. Neta~\cite{non-station} proposed the bespoke Non-Stationary Solvers, a training-based faster solver for both FDMs and DPMs by changing the FDMs to the DPMs. Due to the dramatic change in the representation of the probabilistic path in FDMs, the acceleration method of the ODE of the FDMs remains unexplored compared to the DPMs.

In this paper, we try to offer a new solution to achieve training-free acceleration for the FDMs. We formulate our acceleration method based on the previous steps. This simplifies a lot of complex factors. 
\section{Preliminary}
\label{sec:preliminary}
Flow diffusion models (FDMs) use a continuous time process, such as flow models with ODE~\cite{flow_for_generative,rectified_flow}, to connect the Gaussian probability paths between data distribution and the Gaussian noise distribution as follows:
\begin{equation}
    dx_{t} = v_{\theta}(x_{t},t)dt,
    \label{eq:flow_reverse}
\end{equation}
where $v_{\theta}(\ast,\ast)$ is the estimation modeled by a neural network for the velocity function. 

Now, we define the time interval $[1,0]$, where $t=1$ means the sample is the pure Gaussian noise and $t=0$ means the sample is back to the data distribution. Given $x_{1}$ sampled for the Gaussian noise distribution, we generate the $x_0$ from real data distribution as
\begin{equation}
    x_{0} = x_{1} + \int_{1}^{0}v_{\theta}(x_{s},s)ds,
    \label{eq:flow_general}
\end{equation}
There is a continuous integral that needs to be calculated. We first introduce two solvers for ODE.

\textbf{Euler solver.} The Euler solver is a classical solver for ODE. It first splits $[1,0]$ into the sub-sequence$[1,..., t_{n-1}, t_{n},...,0]$. Then, the Euler solver uses the iteration way to solve the ODE in each time step as:
\begin{equation}
    x_{t_{n}} = x_{t_{n-1}} + \int_{t_{n-1}}^{t_{n}} v_{\theta}(x_{s},s)ds.
    \label{eq:sub-sequance}
\end{equation}
Eq.~\ref{eq:sub-sequance} could be approximated as:
\begin{equation}
    x_{t_{n}} = x_{t_{n-1}} + h_{n}v_{\theta}(x_{t_{n-1}},t_{n-1}),
    \label{eq:euler1}
\end{equation}
where $h_{n}$ is the time interval between the sub-sequence of $t_{n-1}$ and $t_{n}$. 

\textbf{Heun solver.} The Heun solver uses $t_{n}$ and $t_{n-1}$ two points to reduce the approximation error in continuous integral.
It adds an additional function evaluation of $v_{\theta}(\ast,\ast)$ for the future step $t_{n}$ by leveraging the Euler solver to obtain an approximation $\hat{x}_{t_{n}}$ for $x_{t_{n}}$. The Heun solver improved the Euler solver as:
\begin{equation}
     x_{t_{n}} = x_{t_{n-1}} + \frac{1}{2}h_{n}(v_{\theta}(x_{t_{n-1}},x_{t_{n-1}})+v_{\theta}(\hat{x}_{t_{n}},t_{n})),
     \label{eq:heun}
\end{equation}
where
\begin{equation}
    \hat{x}_{t_{n}} = x_{t_{n-1}} + h_{n}v_{\theta}(x_{t_{n-1}},t_{n-1}).
\end{equation}

\textbf{Our motivation.} For Euler solver, the approximation errors of Eq.~\ref{eq:euler1} for Eq.~\ref{eq:sub-sequance} is $O(h_{n})$~\cite{rectified_flow}. This requires the Euler solver to keep the short time interval $h_{n}$, which naturally needs many time steps to generate high-quality images. Each time step needs to evaluate $v_{\theta}(\ast,\ast)$ once, which requires many NFE. 

The Heun solver leverages the results of the Euler solver, i.e., $\hat{x}_{t_{n-1}}$, to reduce the approximation errors for $x_{t_{n}}$ by two function evaluations of $v_{\theta}(\ast,\ast)$. This makes the approximation errors for Eq.~\ref{eq:sub-sequance} reduce to $O(h_{n}^{2})$~\cite{rectified_flow}, where $h_{n}<1$. Under the same $h_{n}$, the approximation errors for the Heun solver are always smaller than the Euler solver, thus improving the generation quality~\cite{dpm_solver}. But the Heun solver needs two NFE in each time step. 

Hence, given a new solver, if it could satisfy: 1) the better approximation error, e.g., $O(h_{n}^{p})$ where $p>1$, and 2) the one function evaluation in each time step, it could accelerate the ODE and generate high-quality images. For example, assume the approximation error bound is $O(0.25)$. The Euler solver needs $\text{NFE}=4$ by setting $h_{n}=0.25$ if we directly split the time interval to the equal length sub-interval. Please note that $[1,0]$ will be divided into four sub-intervals and $\text{NFE}=1$ for each sub-interval. 
Under the same time-interval splitting strategy, the Heun solver only needs two sub-intervals, but $\text{NFE}=2$ for each sub-interval. This leads to $\text{NFE}=4 \ (2 \times 2)$ for the Heun solver too. The new solver with $p=2$ only needs $\text{NFE}=2 \ (2 \times 1)$, where it needs two sub-intervals and $\text{NFE}=1$ for each sub-interval. In this condition, this new solver successfully reduces the NFE. To achieve these, different from the Heun solver, which uses the future step,  the previous steps are a good choice motivated by the DPMs.

\section{Methodology}
\label{sec:method}

In this paper, we propose our flow-solver. Assuming that the overall number of time steps is $T$, our flow-solver aims to achieve better approximation errors to $O(h_{n}^{p})$ and also keep $\text{NEF}=1$ in each time step, where $p$ is the number of previous steps. Concretely, motivated by the DPMs~\cite{dpm_solver}, we find an interesting formulation for the Eq.~\ref{eq:sub-sequance} by introducing the previous evaluation results of $v_{\theta}(\ast,\ast)$. Based on this, we could directly leverage the Taylor expansion to approximate the continued integral, where the approximation errors could be bounded at $O(h_{n}^{p})$ and only need one NEF in each step.

Concretely, we assume that $v_{\theta}(\ast,\ast)$ exists the $p$-th order derivatives. Then, we first leverage the Taylor expansion to connect the high-order derivatives and the continuous integral as the following lemma. 
\begin{lemma}
    The continuous integral in Eq.~\ref{eq:sub-sequance} could be solved by the high-order derivative of $v_{\theta}(\ast,\ast)$ as follows:
    \begin{equation}
            x_{t_{n}} = x_{t_{n-1}} + \sum_{i=0}^{p}C_i \frac{v^{i}_{\theta}(x_{t_{n-1},t_{n-1}})}{i!} + O(h_{n}^{p}), 
            \label{eq:integral}
    \end{equation}
    where $p$ is the order for the solver. $v_{\theta}(x_{t_{n-1},t_{n-1}})$ is the $i$-th order partial derivatives of $t$ for $v_{\theta}(x_{t_{n-1},t_{n-1}})$ at $t_{n-1}$, where $v_{\theta}^{0}(\ast,\ast) = v_{\theta}(\ast,\ast)$.
    $C_i$ is the factor and could be calculated as:
    \begin{equation}
        C_i = \frac{(t_{n}-t_{n-1})^{i+1}}{i+1}.
    \end{equation}
    \label{lemma:discrete}
\end{lemma}
The detailed proofs are in the supplementary.
Lemma~\ref{lemma:discrete} indicates that the key is how to approximate the terms of the high-order derivative. Specifically, we split $i=0$ for the sum operator in Lemma~\ref{lemma:discrete} as:
\begin{equation}
        x_{t_{n}} = x_{t_{n-1}} + h_{n}v_{\theta}(x_{t_{n-1}},t_{n-1}) +\sum_{i=1}^{p}C_{i}v^{i}_{\theta}(x_{t_{n-1},t_{n-1}}).
\label{eq:target}
\end{equation}
Eq.~\ref{eq:target} is a general formulation. The Euler solver chooses to ignore the high-order derivatives term since by canceling the sum operator, i.e., setting $p=0$, $x_{t_{n-1}} + h_{n}v_{\theta}(x_{t_{n-1}},t_{n-1})$ is exactly the Euler solver. The Heun solver chooses to leverage the future steps. Our fast-solver tries to approximate high-order derivatives in Eq.~(\ref{eq:target}), that is to approximate the $\sum_{i=1}^{p}C_i v^{i}_{\theta}(x_{t_{n-1},t_{n-1}})$ by leveraging the previous steps.  

Formally, to use the other step, e.g. $x_{t_{n}}$, to approximate the high-order derivatives, we redefine $v_{\theta}(x_{t_{n}}, t_{n})$ as the approximation for $t_{n}$ at $t_{n-1}$ by implementing the Taylor expansion. We have:
\begin{equation}
    v_{\theta}(x_{t_{n}}, t_{n}) = \sum_{i=0}^{p}v^{i}_{\theta}(x_{t_{n-1}}, t_{n-1})\frac{(t_{n}-t_{n-1})^{i}}{i!},
\end{equation}
where $\frac{(t_{n}-t_{n-1})^{i}}{i!}$ is the parameter of Taylor expansion. Then, by splitting $i=0$ out of the sum operation, we have:

\begin{equation}
    v_{\theta}(x_{t_{n}}, t_{n}) - v_{\theta}(x_{t_{n-1}},t_{n-1}) = \\\sum_{i=1}^{p} a_{i}v^{i}_{\theta}(x_{t_{n-1}}, t_{n-1}),
    \label{eq:evi}
\end{equation}
where $a_{i} = \frac{(t_{n}-t_{n-1})^{i}}{i!}$. In this way, we could find that $v_{\theta}(x_{t_{n}}, t_{n}) - v_{\theta}(x_{t_{n-1}},t_{n-1})$ derives the high-order derivative we need in Eq.~\ref{eq:target}.

\textbf{Leveraging previous steps.} Since $t_{n}$ is the future step of the $t_{n-1}$, which is unknown, thus we use the previous steps to approximate the high-order order derivative in Eq.~\ref{eq:target}. We replace the $t_{n}$ to the previous step $t_{n-1-m}$, which could be defined as follows:
\begin{equation}
    D_{m} = v_{\theta}(x_{t_{n-1-m}}, t_{n-1-m}) - v_{\theta}(x_{t_{n-1}}, t_{n-1}),
    \label{eq:all_start_here}
\end{equation}
where $t_{n-1-m} < t_{n-1}$ is the previous $m$ steps related to the $t_{n-1}$. This also enables us to create a cache to save the function evaluation results of previous steps for $t_{n-1-m}$, and thus reduce the number of NFEs for each sub-interval.  

However, Eq.~\ref{eq:all_start_here} exists the mismatch of the coefficients, which hinders it from directly approximating the $p$-th order derivative. To overcome this, we introduces a factor $B_{m}$ for $D_{m}$ as $B_{m}D_{m}$. Suppose that we have previous $k$ steps, by Taylor expansion at $p$-th order w.r.t $t_{n-1}$ for $D_{m}$, we have:
\begin{equation}
\begin{split}
         \sum_{m=1}^{k}B_{m}D_{m} &\approx \sum_{m=1}^{k}B_{m}\sum_{i=1}^{p}\frac{(h_{n-1-m})^iv_{\theta}^{i}(x_{t_{n-1}},t_{n-1})}{i!} \\
         &= \sum_{i=1}^{p}\frac{v_{\theta}^{i}(x_{t_{n-1}},t_{n-1})}{i!}\sum_{m=1}^{k}B_{m}(h_{n-1-m})^i,
\end{split}
\end{equation}
where $h_{n-1-m} = t_{n-1-m}-t_{n-1}$. Now, by $\sum_{m=1}^{k}B_{m}(h_{n-1-m})^i = C_{i}$, we create a polynomial interpolation problem~\cite{unipc}. There is a closed-form only when $k=p$ (this is why we use multiple previous steps), meaning the highest $p$-th order derivative we could approximate equals the $k$. This also explains why the Heun solver could only approximate $1$-th order solver due to the mismatch coefficients when $p>1$. The flow-solver could solve it by the $p$ number of previous steps (i.e., $k=p$). In this condition, this polynomial interpolation problem has a closed-form based on the matrix $R_{p}$~\cite{unipc}. Therefore, the following lemma could be used to formulate our flow-solver based on the closed-form of the polynomial interpolation problem.
\begin{lemma}
   The flow-solver could be defined based on $B_{m}D_{m}$ as follows:
    \begin{equation}
        x_{t_{n}} = x_{t_{n-1}} + h_{n}v_{\theta}(x_{t_{n-1}},t_{n-1}) + \sum_{m=1}^{p} B_{m}D_{m} ,
        \label{eq:unic_flow}
    \end{equation}
    where $p$ is the number of the previous steps, which is equal to the $p$-th derivatives it could be estimated. The $B_{m}$ is the coefficient, which is equal to:
    \begin{gather}
        B_{m} = S_{p,m} \\
        S_{p} =
         R_{p}^{-1} \begin{bmatrix}
            C_{1}\\
            ...\\
            C_{p}
        \end{bmatrix}\\
        R_{p} = \begin{bmatrix}
             1&t_{n-2}-t_{n-1}  &...& (t_{n-2}-t_{n-1})^{p-1}\\
             1&t_{n-3}-t_{n-1} &...& (t_{n-3}-t_{n-1})^{p-1} \\
             ...&... &...& ... \\
            1&t_{n-1-p}-t_{n-1} &...& (t_{n-1-p}-t_{n-1})^{p-1}
        \end{bmatrix}.
    \end{gather}
    $S_{p,m}$ is the $m$-th row and $p$ is the number of rows since $p$-th derivatives generates the number of $p$ coefficients. 
    \label{lemma:correction term}
\end{lemma}
The detailed proof is in the supplementary. As discussed below, the flow-solver could approximate the $p$-th order derivative and need only one NFE in each step.

\textbf{One function evaluation:} Lemma~\ref{lemma:correction term} leverages the previous $p$ steps results of $v_{\theta}(\ast,\ast)$ to replace the further step $t_{n}$. In this way, we could create a simple cache to collect the $p$ number of the previous results of the function evaluations. Leveraging the cache, we could approximate $t_{n}$ step based on the $t_{n-1}$ with only querying the $v_{\theta}(\ast,\ast)$ once for $v_{\theta}(x_{t_{n-1}},t_{n-1})$.

\textbf{Approximation error:}
Lemma~\ref{lemma:correction term} enables us to formulate a $p$-order solver since we could introduce the number of $p$ previous steps to estimate the $p$-th order derivative.
In the end, we have the following lemma to show the bound approximation errors.
\begin{lemma}
    Given the $p$-th flow-solver in Lemma~\ref{lemma:correction term} at $t_{n-1}$ timesteps, the approximation errors for the next step $t_{n}$ are equal to $O(h_{n}^{p})$.
    \label{lemma:approximation}
\end{lemma}
The detailed proof is in the supplementary. Lemma~\ref{lemma:approximation} bound the approximation errors to the $O(h_{n}^{p})$, which is similar to the Heun solver. Assuming we have $T$ time steps, the difference is $\text{NFE}=T$ for the flow-solver instead of $\text{NFE}=2T$ for the Heun solver. In this condition, we successfully achieve the acceleration for the sampling process.  
Meanwhile, there is no need to calculate the future step. Since $h_{n}<1$, Lemma~\ref{lemma:approximation} also proves that the flow-solver has fewer approximation errors.

\begin{algorithm}
    \caption{The algorithm for flow-solver} \label{al:flow-solver}
    \begin{algorithmic}[1]
     \Statex \textbf{Input:} $v_{\theta}(\ast,\ast)$, a buffer $\text{buff} = \{ v_{\theta}(x_{t_{n-m}},t_{n-m})\}^{p}_{m=2}$, $x_{t_{n-1}}$.
     \Statex \textbf{Output:} $x_{t_{n}}$.
        \State Trick $\leftarrow$ True. \Comment{Implement trick}
        \State Calculate $v_{\theta}(x_{t_{n-1}},t_{n-1})$.
        \State Calculate $B_{m}$. \Comment{Lemma ~\ref{lemma:correction term}}
        \If{Trick}
            \State Push $v_{\theta}(x_{t_{n-1}},t_{n-1})$ into buff.
            \State Calculate $D_{m}$. \Comment{Lemma ~\ref{lemma:correction term}}
            \State $x_{t_{n-1}} \leftarrow x_{t_{n-1}} + \sum_{m=1}^{p}B_{m}D_{m}$.
            \State Pop $v_{\theta}(x_{t_{n-1}},t_{n-1})$ out buff.
        \EndIf
        \State Calculate $D_{m}$. \Comment{Lemma ~\ref{lemma:correction term}}
        \State $x_{t_{n}} \leftarrow x_{t_{n-1}} + h_{n}v_{\theta}(x_{t_{n-1}},t_{n-1}) + \sum_{m=1}^{p-1}B_{m}D_{m}$.
    \Statex\textbf{Return:} $x_{t_{n}}$.
    \end{algorithmic}
\end{algorithm}

\subsection{Increasing the solver order} In the end, we show a trick in DPMs~\cite{unipc} to further increase the approximation errors for the flow-solver to $O(h_{n}^{p+1})$ called the predictor-corrector trick, which is equal to introduce $p+1$ steps. This could save little memory and improve the generation quality. Concretely, after obtaining $x_{t_{n}}$, the result of $v_{\theta}(x_{t_{n}},t_{n})$ could be put into the Lemma~\ref{lemma:correction term}. In this way, we could leverage $v_{\theta}(x_{t_{n}},t_{n})$ to get a more precise $x_{t_{n-1}}$, which will benefit the prediction for $x_{t_{n+1}}$ and has the following lemma.
\begin{lemma}
    Given the $p$-th flow-solver at $t_{n-1}$ timesteps, the approximation errors for the next step $t_{n}$ are increasing from $O(h_{n}^{p})$ to $O(h_{n}^{p+1})$ by introducing predictor-corrector trick.
    \label{lemma:approximation_trick}
\end{lemma}
The detailed proof is in the supplementary. Algorithm~\ref{al:flow-solver} summarized the overall process for flow-solver. To sum up, we propose a novel flow-solver for FDMs. Our flow-solver derives from a unique formulation of the connection among high-order derivatives, the Euler and Heun solvers. The flow-solver could approximate the high-order derivatives by formulating the $D_{m}$ by introducing the previous $p$ steps results to replace the further step, thus avoiding additional query for $v_{\theta}(\ast,\ast)$. The approximation errors of the flow-solver are at least bound at $O(h_{n}^{p})$. Thus, the flow-solver enables the large time step while alleviating the approximation errors.

\textbf{A unified perspective of Euler solver.} Our flow-solver also provides a new perspective of the Euler solver. When $p=1$, our solver from Lemma~\ref{lemma:correction term} is:
\begin{equation}
     x_{t_{n}} = x_{t_{n-1}} + h_{n}v_{\theta}(x_{t_{n-1}},t_{n-1}).
     \label{eq:euler}
\end{equation}
Eq.~\ref{eq:euler} is exactly the Euler solver for FDMs~\cite{rectified_flow}, further proving our solver's validity and extensibility.

\textbf{Remark.} Compared to the fast solvers in DPMs, e.g., DPM-solver~\cite{dpm_solver,unipc}, the common point is that we both leverage the previous steps to approximate the next step precisely. The difference is how to match the coefficients for the high-order derivative part. Flow-solver changes it to a polynomial interpolation problem, and flow-solver is formulated based on the closed-form, alleviating the additional computation errors for additional factors. This enables the flow-solver to be used for pixel space and latent space-based FDMs.

\section{Experimental Results}
\label{sec:experiment}
\subsection{Implementation Details}
\textbf{Unconditional generation.} We first test our fast solver on the unconditional generation on the CIFAR-10, CelebA-HQ, LSUN-Church, and LSUN-Bedroom. The resolutions of generated images include $32 \times 32$ and $256\times256$. Following the previous works~\cite{dpm_solver}, we use Fréchet inception distance (FID) to measure the quality of the generated images. To calculate the FID, we generate 30K images for each dataset.  

\textbf{Conditional generation.} We test our fast solver on the conditional generation tasks based on the classifier-free guidance~\cite{classifier-free}. We first test the label-based condition on ImageNet~\cite{ImageNet}, where the resolution is $256\times256$. We use FID to measure the quality by generating the 30K images. The strength for classifier-free guidance (CFG) is set to the default setting~\cite{sit}.

Then, we further test our fast solver on the real text-to-image generation task. We randomly choose the prompts from PartiPrompts~\cite{freedom} and generate the 10K images with $512\times512$ resolution. The ground truth images are generated using the Euler solver with $\text{NFE}=100$. We measure the generation quality based on the FID between the ground truth and the generated images.

\textbf{Pre-trained FDMs.} We use the Rectified flow~\cite{rectified_flow} from the official checkpoint for CIFAR-10, CelebA-HQ, LSUN-Church, and LSUN-Bedroom, where the models are pre-trained in the pixel space. We use the SiT-XL~\cite{sit} for ImageNet, where the model is pre-trained in the latent space. The Stable Diffusion 3.0~\cite{sd3.0} is used from the official checkpoint. All the experiments are run on a single RTX 4090 GPU.

\textbf{Baselines.} Following the previous works~\cite{dpm_solver}, we choose 
the Euler solver, the Heun solver~\cite{sd3.0}, and the RK-3, a black-box ODE solver (RK-3)~\cite{sde} as the baselines. To keep a fair comparison, we ensure that all the baselines and the flow-solver have the same NEFs. 

\subsection{Qualitative Results}
\begin{table}
    \centering
    \caption{Unconditional generation results on CIFAR-10. We compared the FID $\downarrow$ among different sampling solvers for FDMs with different NEFs.}
    \begin{tabular}{ccccc}
    \hline \multirow{2}{*}{ Sampling Method } & \multicolumn{4}{c}{ NFE } \\
    \cline { 2 - 5 } & 7 & 8 & 9 & 10 \\
    \hline Euler solver& 23.43 & 19.23 & 16.07 & 13.79 \\
    Heun solver&  46.35 &  40.07 &  33.44 &  30.03 \\
    RK-3 & 175.03 & 91.12 &83.33 & 77.37 \\
    Flow-solver (our)  & \textbf{8.93} & \textbf{7.70}& \textbf{7.27}& \textbf{6.62}\\
    \hline
    \end{tabular}
    \label{tab:cifar}
\end{table}

\begin{table}
    \centering
    \caption{Unconditional generation results on CelebA-HQ. We compared the FID $\downarrow$ among different sampling solvers for FDMs with different NEFs.}
    \begin{tabular}{ccccc}
    \hline \multirow{2}{*}{ Sampling Method } & \multicolumn{4}{c}{ NFE } \\
    \cline { 2 - 5 } & 7 & 8 & 9 & 10 \\
    \hline Euler solver& 113.13 & 106.87 & 100.76 & 96.54 \\
        Heun solver&  291.38 &  285.76 &  287.78 &  292.55 \\
    RK-3 &   307.31 & 290.22 &  281.37 &   279.09 \\
    Flow-solver (our)  &   \textbf{81.14} & \textbf{74.86} &  \textbf{70.82} & \textbf{67.28} \\
    \hline
    \end{tabular}
    \label{tab:celehq}
\end{table}

\begin{table}
    \centering
    \caption{Unconditional generation results on LSUN-Bedroom. We compared the FID $\downarrow$ among different sampling solvers for FDMs with different NEFs.}
    \begin{tabular}{ccccc}
    \hline \multirow{2}{*}{ Sampling Method } & \multicolumn{4}{c}{ NFE } \\
    \cline { 2 - 5 } & 7 & 8 & 9 & 10 \\
    \hline Euler solver&  43.49 &  35.38 & 30.36 & 26.48 \\
        Heun solver&    69.59 &  59.09 &  49.15 &  42.46 \\
    RK-3 &    303.26 & 363.54 &  370.30 &   359.14 \\
    Flow-solver (our)  & \textbf{16.20} &  \textbf{14.59} & \textbf{14.13} & \textbf{13.80} \\
    \hline
    \end{tabular}
    \label{tab:lsun-bedroom}
\end{table}

\begin{table}
    \centering
    \caption{Unconditional generation results on LSUN-Church. We compared the FID $\downarrow$ among different sampling solvers for FDMs with different NEFs.}
    \begin{tabular}{ccccc}
    \hline \multirow{2}{*}{ Sampling Method } & \multicolumn{4}{c}{ NFE } \\
    \cline { 2 - 5 } & 7 & 8 & 9 & 10 \\
    \hline Euler solver&  84.47 &   67.37 & 55.64 & 46.64 \\
        Heun solver&  128.18 &  102.99 &  92.77 &  83.58\\
    RK-3 &  333.26 &  290.12 &   304.24 &   284.01 \\
    Flow-solver (our)  & \textbf{ 34.32} &  \textbf{ 27.33} & \textbf{22.75} & \textbf{19.49} \\
    \hline
    \end{tabular}
    \label{tab:lsun-church}
\end{table}

\begin{table}
    \centering
    \caption{Conditional generation results on ImageNet. We compared the FID $\downarrow$ among different sampling solvers for FDMs with different NEFs.}
    \begin{tabular}{ccccc}
    \hline \multirow{2}{*}{ Sampling Method } & \multicolumn{4}{c}{ NFE } \\
    \cline { 2 - 5 } & 7 & 8 & 9 & 10 \\
    \hline Euler solver&  84.16 &  73.86 & 67.07 &  61.55 \\
    Heun solver&  115.90 &  102.99 &  92.77 &  83.58 \\
    RK-3 &   96.44 & 96.33 &  97.77 &   103.27 \\
    Flow-solver (our)  & \textbf{44.91} & \textbf{42.09} & \textbf{39.69} & \textbf{38.29} \\
    \hline
    \end{tabular}
    \label{tab:imagenet}
\end{table}

\textbf{Unconditional generation.} We first prove the efficiency of the flow-solver in the unconditional generation. Based on the four datasets, we report the comprehensive qualitative results with different numbers of NFEs shown in Table~\ref{tab:cifar}-\ref{tab:lsun-church}. Concretely, under the CIFAR-10, the flow-solver improves $51.05\%$, $54.39\%$, $59.38\%$, and $62.69\%$ FID with $\text{NEF}=7$, $\text{NEF}=8$, $\text{NEF}=9$, and $\text{NEF}=10$ compared to the SOTA baselines. Compared to the Euler solver, the flow-solver performs under $\text{NEF}=7$ better than the Euler solver under $\text{NEF}=10$. Then, compared to the Heun solver and RK-3, the flow-solver greatly improves under fewer NEF. Specifically, compared to the RK-3, the black-box high-order solver, the flow-solver significantly improves the FID and maintains high-quality generation. This verifies the potential of the flow-solver.
\begin{figure*}
    \centering
    \begin{subfigure}[b]{0.3\textwidth}
    \includegraphics[width=\textwidth,scale=1.1]{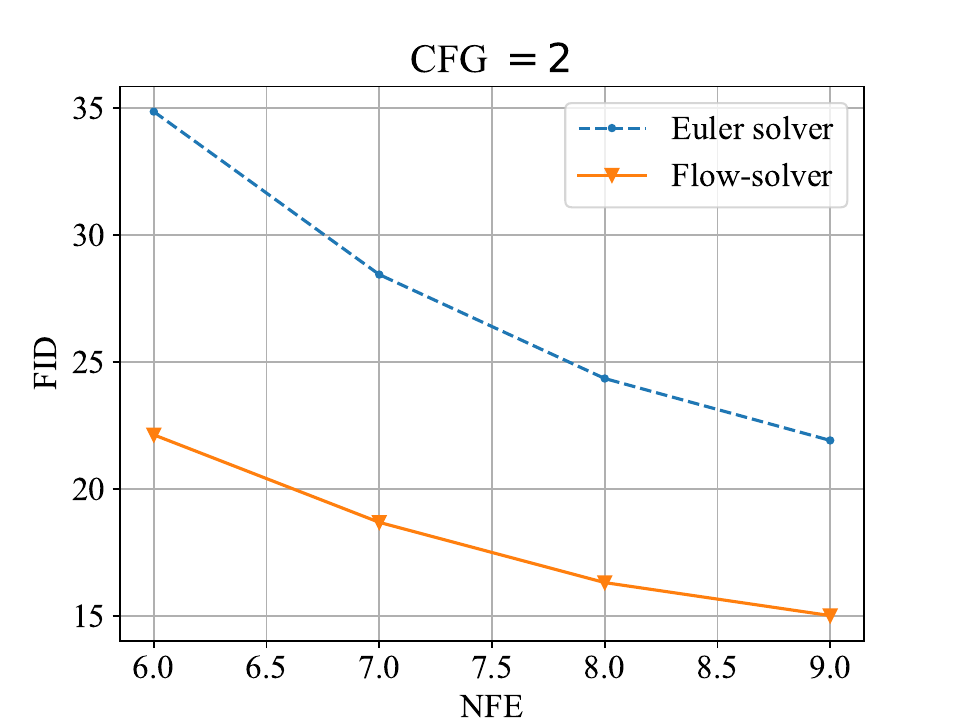}
    \label{fig:1}
    \end{subfigure}
    \begin{subfigure}[b]{0.3\textwidth}
    \includegraphics[width=\textwidth,scale=1.1]{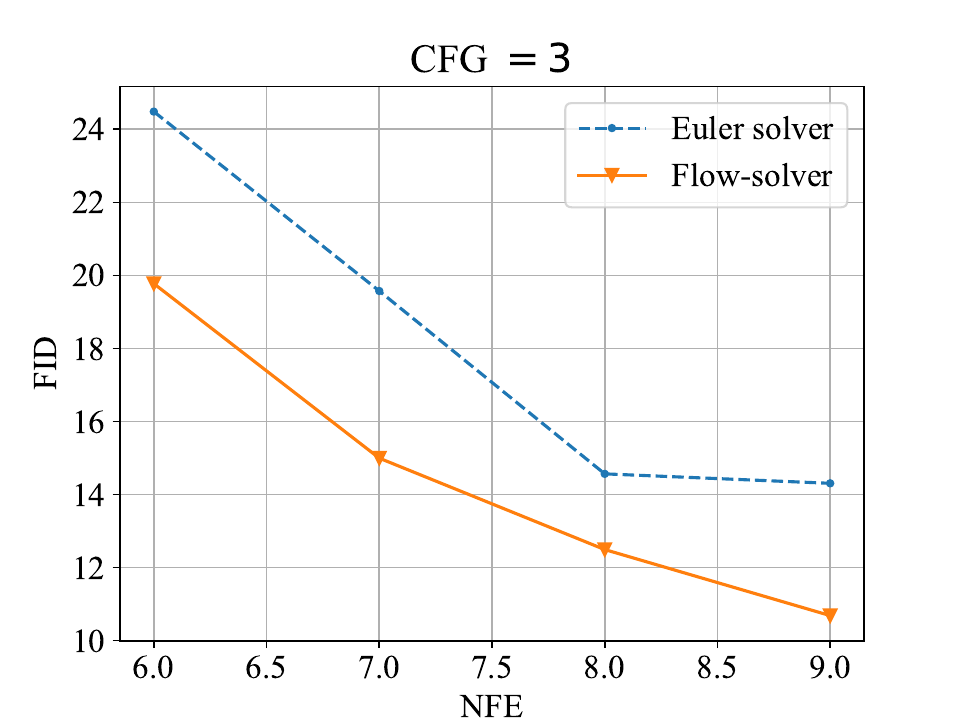}
    \label{fig:1}
    \end{subfigure}
    \begin{subfigure}[b]{0.3\textwidth}
    \includegraphics[width=\textwidth,scale=1.1]{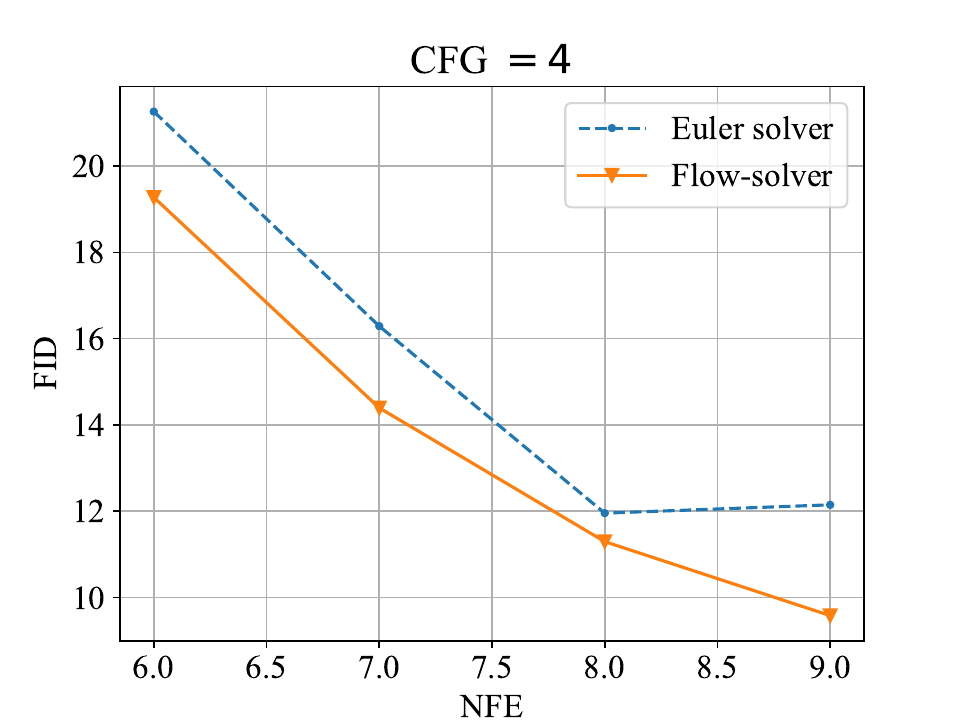}
    \label{fig:1}
    \end{subfigure}
    \caption{Conditional generation results on real text-to-image generations. We compared the FID $\downarrow$ among different sampling solvers for FDMs with different NEFs. We also report the influence of the different strengths of classifier-free guidance (CFG).}
    \label{fig:cfg}
\end{figure*}

We then further verify the efficiency of the flow-solver by tacking the high-resolution datasets and increasing the resolution from $32 \times 32$ to $256 \times 256$. Specifically, under CelebA-HQ, the flow-solver keeps the best performance. Compared to the Euler solver, the flow-solver improves an average of $29.56\%$ FID within different NEFs. Then, it can be noticed that the Heun solver and RK-3, two high-order baselines, show an unstable generation while the FID dramatically increases. The flow-solver avoids unstable generation due to the coefficient match, thus further proving its efficiency and validity.

To further explore the efficiency of the flow-solver and ensure stable generation, we report the additional datasets: LSUN-Bedroom and LSUN-Church with $256 \times 256$ resolution. Concretely, the flow-solver improves by an average $55.71\%$, and $59.03\%$ FID under LSUN-Bedroom and LSUN-Church, respectively, compared to the SOTA baseline, which first proves the efficiency of the flow-solver. The unstable generation happens for high-order solvers such as RK-3, and the flow-solver still works well. These results demonstrate the validity of the flow-solver.

To sum up, the flow-solver significantly improves FID in CIFAR-10, CelebA-HQ, LSUN-Bedroom, and LSUN-Church. Since the SOTA baseline is the Euler solver ( This also explains why the Euler solver is the main paradigm for the FDMs), we report qualitative examples in LSUN-Bedroom and LSUN-Church, as shown in Fig.~\ref{fig:uncondition} to further show how improvement made by the flow-solver. It can be noticed that the flow-solver could generate higher-quality images than the Euler solver. These results prove the efficiency of the flow-solver. Additionally, the models for three datasets are trained in the pixel space from resolution $32\times32$ to $256\times256$, proving the flow-solver's validity in pixel space and is compatible under different resolutions. In the end, the unstable generation that happened in the high-order solvers of the baseline is overcome by the flow-solver, which further proves the validity of the flow-solver.
\begin{figure*}
    \centering
    \includegraphics[width=\textwidth]{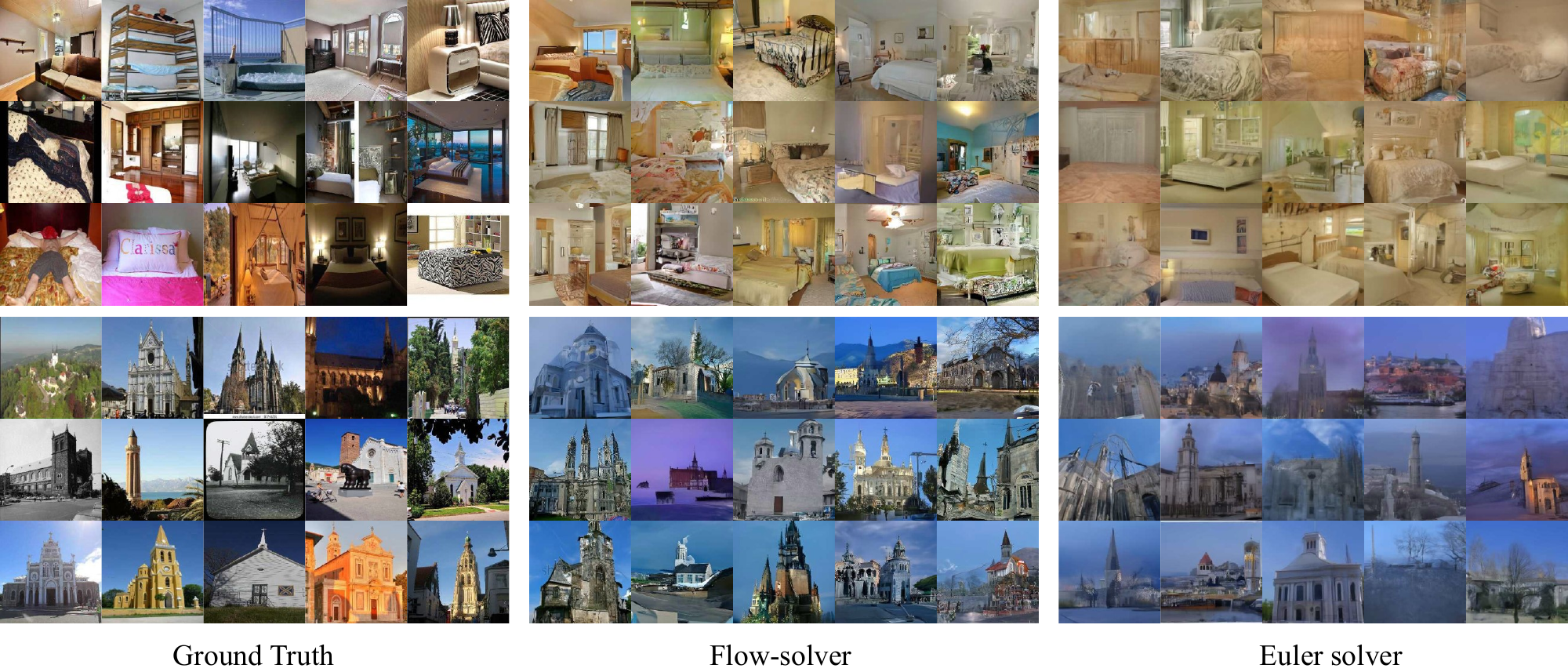}
    \caption{Qualitative examples for the unconditional generation. We mainly report the results on LSUN-Bedroom and LSUN-Church under with $\text{NEF}=10$.}
    \label{fig:uncondition}
\end{figure*}

\textbf{Conditional generation.} To prove the efficiency of the flow-solver on the conditional generation. We test its performance under the class label and text prompt conditions. Firstly, we report the detailed experiential results for the class label-based conditional generation tasks on ImageNet shown in Table~\ref{tab:imagenet}. Specifically, the flow-solver increases the FID from 84.16, 73.86, 67.07, and 61.55 to 44.91, 42.09, 39.69, and 38.29 with the rising NEF from 7 to 10, respectively, compared to the SOTA baseline. This improvement first proves the efficiency of the flow-solver. Then, we could find that the high-order solvers do not show unstable generation, and the flow-solver achieves the SOTA performance compared to them. These results prove the validity of the flow-solver.  
\begin{figure*}
    \centering
    \includegraphics[width=\textwidth]{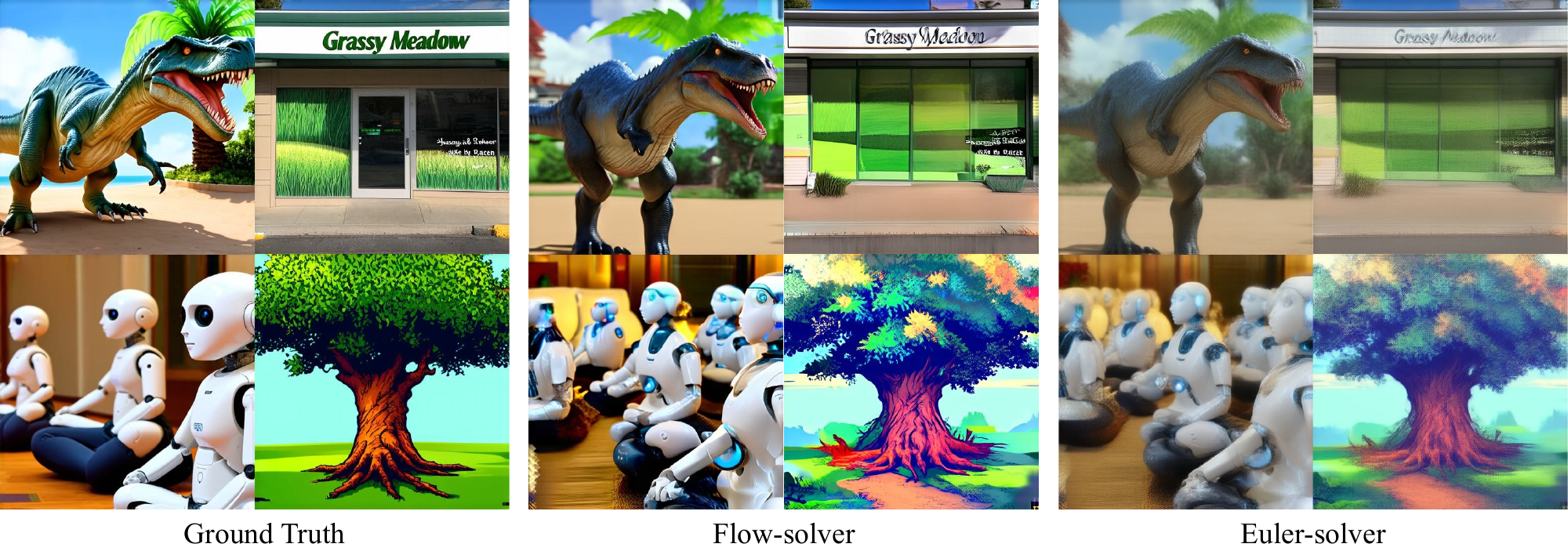}
    \caption{Qualitative examples for the conditional generation. We mainly report the text-to-image generation under different prompts with $\text{NEF}=6$ and $\text{CFG}=2$.}
    \label{fig:t2i_results}
\end{figure*}

We then implement the flow-solver to real text-to-image generation tasks based on Stable Diffusion 3.0 to explore its potential. We report the FID among different combinations of NFE and CFG. The flow-solver achieves the SOTA performance under $\text{NFE} \in [6,7,8,9]$ and $\text{CFG} \in [2,3,4]$ shown in Fig.~\ref{fig:cfg}. Concretely, with $\text{CFG}=2$, the flow-solver significantly improves FID compared to the Euler solver. Then, with the increase of the CFG from 2 to 4, the flow-solver still outperforms the Euler solver in different combinations between CFG and NFE. Then, in all combinations between NEF and CFG, the flow-solver outperforms the Heun solver. These results show that the efficiency of the flow-solver. Meanwhile, with the CFG increase, the flow-solver's FID does not show dramatic increases. This shows that the flow-solver alleviates the unstable generation under high CFG and further proves the validity of the flow-solver.

To sum up, the flow-solver achieves the SOTA performance on the conditional generation. The FDMs we used are pre-trained in the latent space. These results prove that the flow-solver could also work well under the latent space. Meanwhile, the resolution of text-to-image generation is further increased to $512\times 512$. The flow-solver does not show the unstable phenomenon in the CelebA-HQ, which proves the validity of the flow-solver. As the mainstream solver, the Euler solver still outperforms the Heun solver. To further explore the potential of the flow-solver compared to the Euler solver, we report the additional qualitative examples shown in Fig.~\ref{fig:t2i_results}. It can be found that the flow-solver could generate images with more detail under different prompts than the Euler solver. This is why the flow-solver could increase the FID and improve the generation quality.

\subsection{Ablation Study}
\textbf{Influence of $p$.} We make an ablation study to explore the influence of the order $p$ (i.e., the previous steps we have introduced). The comparison baselines include: 1) ($p=1$), which means we introduce no previous steps. 2) ($p=2, \text{w/o}$) and ($p=3, \text{w/o}$), where we drop the corrector trick since the corrector trick could further increase the order to $p+1$. This aims to prove that increasing the order could benefit the quality of generation. 3) ($p=2$) with corrector trick. The corrector trick could increase the order to $p+1$. This aims to prove that the corrector trick could correctly increase the order. 

We report the experimental results in Table~\ref{tab:albation_increasing}. Specifically, we could find that the flow-solver ($p=3, \text{w/o}$) outperforms the flow-solver ($p=2, \text{w/o}$), which proves that introducing the previous steps could benefit the generation quality. Then, the flow-solver ($p=2$) performs close to the flow-solver ($p=3, \text{w/o}$), proving the corrector trick's validity. Meanwhile, the flow-solver ($p=2$) even improves slightly compared to the flow-solver ($p=3, \text{w/o}$). This is because we ignore the inherent approximation errors from the model itself, which is also verified in DPMs~\cite{unipc}. Thus, to avoid introducing additional approximation errors inherent in the model, we use the flow-solver ($p=2$) with the corrector trick in all experiments. 

\begin{table}
    \centering
    \caption{Ablation study of $p$ on CIFAR-10. We compared the FID $\downarrow$ among different orders for flow-solver with different NEFs, where ($p=i$) means that we introduce previous $i$ steps. w/o means not using the corrector trick.}
    \begin{tabular}{ccccc}
    \hline \multirow{2}{*}{ Sampling Method } & \multicolumn{4}{c}{ NFE } \\
    \cline { 2 - 5 } & 7 & 8 & 9 & 10 \\
    \hline  Flow-solver ($p=1$)& 23.43 & 19.23 & 16.07 & 13.79 \\
    Flow-solver ($p=2$, w/o) & 15.32 & 12.69 & 11.26 &  9.69\\
    Flow-solver ($p=3$, w/o) & \textbf{8.74} & 7.81 &  7.33 & 6.75 \\
    Flow-solver ($p=2$) & 8.93 & \textbf{7.70}& \textbf{7.27}& \textbf{6.62}\\
    \hline
    \end{tabular}
    \label{tab:albation_increasing}
\end{table}

\section{Conclusion}
\label{sec:conslution}
In this paper, we proposed a novel training-free acceleration method for FDMs called flow-solver. The flow-solver is derived from a special formulation from the connection among the high-order derivatives, the Euler solver and the Heun solver. This enables the flow-solver to approximate the continuous integral more precisely by introducing the results of previous steps to avoid querying the model. The experimental results show that the flow-solver outperforms the Euler solve in the conditional and unconditional generation tasks. Meanwhile, the results also show that the flow-solver could work on different resolutions, pixel space-based FDMs, and latent space-based FDMs.

\textbf{Limitations.} As we reported in the ablation study, there are two limitations for the flow-solver except for the large improvement: 1) The inherent approximation errors of models will influence the generation quality. This further hinders the flow-solver from introducing more steps in time to approximate the high-order derivatives as much as possible. This could be improved by increasing the scaler of the FMs. 2) The CFG will increase the inherent approximation errors during the generation process~\cite{classifier-free}. We show more qualitative examples in the supplementary. These defects remain for further work to explore.

\bibliographystyle{unsrt}  
\bibliography{main}  
 \clearpage
\setcounter{page}{1}
\setcounter{lemma}{0}
\subsection{Detailed Proof}
\begin{lemma}
    The continuous integral in Eq.~\ref{eq:sub-sequance} could be solved by the high-order derivative of $v_{\theta}(\ast,\ast)$ as follows:
    \begin{equation}
            x_{t_{n}} = x_{t_{n-1}} + \sum_{i=0}^{k}C_{i}\frac{v^{i}_{\theta}(x_{t_{n-1},t_{n-1}})}{i!},
            \label{eq:integral}
    \end{equation}
    where $k$ is the order for the solver, we set $k=\infty$ to remove the estimation error term temporarily. $v_{\theta}(x_{t_{n-1},t_{n-1}})$ is the $i$-th order partial derivatives for $v_{\theta}(x_{t_{n-1},t_{n-1}})$ at $t_{n-1}$, where $v_{\theta}^{0}(\ast,\ast) = v_{\theta}(\ast,\ast)$.
    $C_{i}$ is the factor and could be calculated as:
    \begin{equation}
        C_{i} = \frac{(t_{n}-t_{n-1})^{i+1}}{i+1}.
    \end{equation}
    \label{lemma:discrete}
\end{lemma}
\begin{proof}
    With the Taylor expansion first time for $v_{\theta}(x_{s},s)ds$ w.r.t $s = t_{n-1}$, we have:
    \begin{equation}
    \begin{split}
        v_{\theta}(x_{s},s)ds &= \sum_{i=0}^{k}\frac{v^{i}_{\theta}(x_{t_{n-1},t_{n-1}})}{i!}(s-t_{n-1})^i\\
        &= \sum_{i=0}^{k}v^{i}_{\theta}(x_{t_{n-1},t_{n-1}})\frac{(s-t_{n-1})^i}{i!}.
    \end{split}
    \end{equation}
    \begin{equation}
    \end{equation}
    Then, $v^{i}_{\theta}(x_{t_{n-1},t_{n-1}})$ is irrelevant to the $s$, we have:
    \begin{equation}
    \begin{split}
    \int_{t_{n-1}}^{t_{n}} v_{\theta}(x_{s},s)ds &= \sum_{i=0}^{k}v^{i}_{\theta}(x_{t_{n-1},t_{n-1}}) \\
    &\int_{t_{n-1}}^{t_{n}}\frac{(s-t_{n-1})^i}{i!}ds.
    \end{split} 
    \end{equation}
    Let $C_{i}=\int_{t_{n-1}}^{t_{n}}\frac{(s-t_{n-1})^i}{i!}ds$, $C_{i}$ now has the closed-form as:
    \begin{equation}
        \begin{split}
            C_{i}&=\int_{t_{n-1}}^{t_{n}}\frac{(s-t_{n-1})^i}{i!}ds\\
            &= \frac{(s-t_{n-1})^{i+1}}{(i+1)i!}\bigg|^{t_{n}}_{t_{n-1}}\\
            &= \frac{(t_{n}-t_{n-1})^{i+1}}{(i+1)!}.\\
            & = \frac{(t_{n}-t_{n-1})^{i+1}}{(i+1)(i)!}.
        \end{split}
    \end{equation}
    Splitting $\frac{1}{i!}$ from $C_{i}$ to match $\frac{v^{i}_{\theta}(x_{t_{n-1},t_{n-1}})}{i!}$, we have:
    \begin{equation}
        C_{i} = \frac{(t_{n}-t_{n-1})^{i+1}}{i+1}.
    \end{equation}
    Thus, we finish the proof.
\end{proof}

\begin{lemma}
   The flow-solver could be defined based on $B_{m}D_{m}$ as follows:
    \begin{equation}
        x_{t_{n}} = x_{t_{n-1}} + h_{n}v_{\theta}(x_{t_{n-1}},t_{n-1}) + \sum_{m=1}^{p} B_{m}D_{m} ,
        \label{eq:unic_flow}
    \end{equation}
    where $p$ is the number of the previous steps, which is equal to the $p$-th derivatives it could be approximated.
    \begin{equation}
        D_{m} = v_{\theta}(x_{t_{m}},t_{m}) - v_{\theta}(x_{t_{n-1}},t_{n-1}).
    \end{equation}
    The $B_{m}$ is the coefficient, which is equal to:
    \begin{gather}
        B_{m} = S_{p,m} \\
        S_{p,m} =
         R_{p}^{-1} \begin{bmatrix}
            C_{1}\\
            ...\\
            C_{p}
        \end{bmatrix}\\
        R_{p} = \begin{bmatrix}
             1&t_{n-2}-t_{n-1}  &...& (t_{n-2}-t_{n-1})^{p-1}\\
             1&t_{n-3}-t_{n-1} &...& (t_{n-3}-t_{n-1})^{p-1} \\
             ...&... &...& ... \\
            1&t_{n-1-p}-t_{n-1} &...& (t_{n-1-p}-t_{n-1})^{p-1}
        \end{bmatrix}.
    \end{gather}
    $S_{p,m}$ is the $m$-th row and $p$ is the number of rows since $p$-th derivatives generates the number of $p$ coefficients. 
    \label{lemma: correction term}
\end{lemma}
\begin{proof}
    Our target is to use $\sum_{m=1}^{p}B_{m}D_{m}$ to estimate $\sum_{i=1}^{k}C_{i}\frac{v^{i}_{\theta}(x_{t_{n-1},t_{n-1}})}{i!}$, the $k$-th order derivative. Firstly, let $p=k$, to represent $\sum_{m=1}^{p}B_{m}D_{m}$ could estimate the $k$-order derivative. Then, by Taylor expansion of $v_{\theta}{x_{t_{m}},t_{m}}$ in $D_{m}$ w.r.t $t_{n-1}$, we have:
    \begin{equation}
    \begin{split}
        D_{m} &= v_{\theta}(x_{t_{m}},t_{m}) - v_{\theta}(x_{t_{n-1}},t_{n-1}) \\
        &= \sum_{i=0}^{k}\frac{v_{\theta}^{i}(x_{t_{n-1},t_{n-1}})}{i!}(t_{m}-t_{n-1})^i - v_{\theta}(x_{t_{n-1}},t_{n-1}) \\
        & = \sum_{i=1}^{k}\frac{v_{\theta}^{i}(x_{t_{n-1},t_{n-1}})}{i!}(t_{m}-t_{n-1})^i.
    \end{split}
        \label{eq:second_taylor}
    \end{equation}
    
    With $B_{m}$ to Eq.~\ref{eq:second_taylor}, we have:
    \begin{equation}
            \sum_{m=1}^{p}B_{m}D_{m}= \sum_{m=1}^{p}B_{m}\sum_{i=1}^{k}\frac{v_{\theta}^{i}(x_{t_{n-1},t_{n-1}})}{i!}(t_{m}-t_{n-1})^i.
    \end{equation}
    Since $p=k$, thus we further have:
    \begin{equation}
    \begin{split}
    \sum_{m=1}^{p}B_{m}D_{m}  &= \sum_{m=1}^{p}B_{m}\sum_{i=1}^{p}\frac{(t_{m}-t_{n-1})^iv_{\theta}^{i}(x_{t_{n-1}},t_{n-1})}{i!} \\
    &=\sum_{m=1}^{p}\sum_{i=1}^{p}B_{m}\frac{v_{\theta}^{i}(x_{t_{n-1},t_{n-1}})}{i!}(t_{m}-t_{n-1})^i \\
 &=\sum_{i=1}^{p}\frac{v_{\theta}^{i}(x_{t_{n-1},t_{n-1}})}{i!}\sum_{m=1}^{p}B_{m}(t_{m}-t_{n-1})^i 
    \end{split}
    \end{equation}

    Now, our target is change to let $\sum_{m=1}^{p}B_{m}(t_{m}-t_{n-1})^i$ be equal to $C_{i}$. This is a polynomial interpolation problem, and thus we have the following formation:
    \begin{equation}
        R_{p} \begin{bmatrix}
            B_{1}\\
            ...\\
            B_{p}
        \end{bmatrix} = \begin{bmatrix}
            C_{1}\\
            ...\\
            C_{p}
        \end{bmatrix},
    \end{equation}
    where $R_{p}$ is the Vandermonde matrix:
    \begin{equation}
        R_{p} = \begin{bmatrix}
            1 & t_{1}-t_{n-1}  &...& (t_{1}-t_{n-1})^{p-1}\\
            1 & t_{2}-t_{n-1} &...& (t_{2}-t_{n-1})^{p-1} \\
            ... & ... &...& ... \\
            1 & t_{p}-t_{n-1} &...& (t_{p}-t_{n-1})^{p-1}
        \end{bmatrix}.
    \end{equation}
    Since $t_{m}$ is the $m$-th previous step thus $t_{p}-t_{n-1}$ is monotonicity, which ensure the invertibility of $R_{p}$. Therefore, we have:
    \begin{equation}
        \begin{bmatrix}
            B_{1}\\
            ...\\
            B_{p}
        \end{bmatrix} = R_{p}^{-1} \begin{bmatrix}
            C_{1}\\
            ...\\
            C_{p}
        \end{bmatrix}.
    \end{equation}
    Thus, we finish the proof. Then, the approximation errors could be notated as $O(h_{n}^{p})$ since the coefficient matching could help the flow-solver to approximate $p$ high-order derivatives of $t_{n}$ at $t_{n-1}$.
\end{proof}

\begin{lemma}
    Given the $p$-th flow-solver in Lemma~\ref{lemma:correction term} at $t_{n-1}$ timesteps, the approximation errors for the next step $t_{n}$ are equal to $O(h_{n}^{p})$.
    \label{lemma:approximation}
\end{lemma}
\begin{proof}
    Leveraging the $B_{m}$, given $p$, we have:
    \begin{equation}
        \sum_{m=1}^{p} B_{m}D_{m} = \sum_{i=0}^{p}v^{i}_{\theta}(x_{t_{n-1}}, t_{n-1})\frac{(t_{n}-t_{n-1})^{i}}{i!} + O(h_{n}^{p}).
    \end{equation}
    Thus, the approximation errors are $O(h_{n}^{p})$ and we finish the proof.
\end{proof}

\begin{lemma}
    Given the $p$-th flow-solver at $t_{n-1}$ timesteps, the approximation errors for the next step $t_{n}$ are increasing from $O(h_{n}^{p})$ to $O(h_{n}^{p+1})$ by introducing predictor-corrector trick.
    \label{lemma:approximation_trick}
\end{lemma}
\begin{proof}
    The predictor-corrector trick~\cite{unipc} aims to add an additional corrector step by leveraging the $v_{\theta}(x_{t},t)$ to enhance $x_{t}$, which benefits the $x_{t+1}$. This could be regarded as introducing $v_{\theta}(x_{t},t)$ into the buff to approximate $p+1$-th order derivative. Thus, the approximation errors are $O(h_{n}^{p+1})$~\cite{unipc} and we finish the proof.
\end{proof}
\subsection{More Quantitative Examples}
\textbf{Comparison among different NEFs.} To show the influence of the different NEFs, we report the ablation study in Table~\ref{tab:ala_NEF}. It could be found that the Heun solver works well under more NFEs than the Euler solver. Meanwhile, these results also prove the potential of the flow-solver since we save half NEF to achieve a similar performance to the Heun solver.

\begin{figure*}
    \centering
    \includegraphics[width=\textwidth]{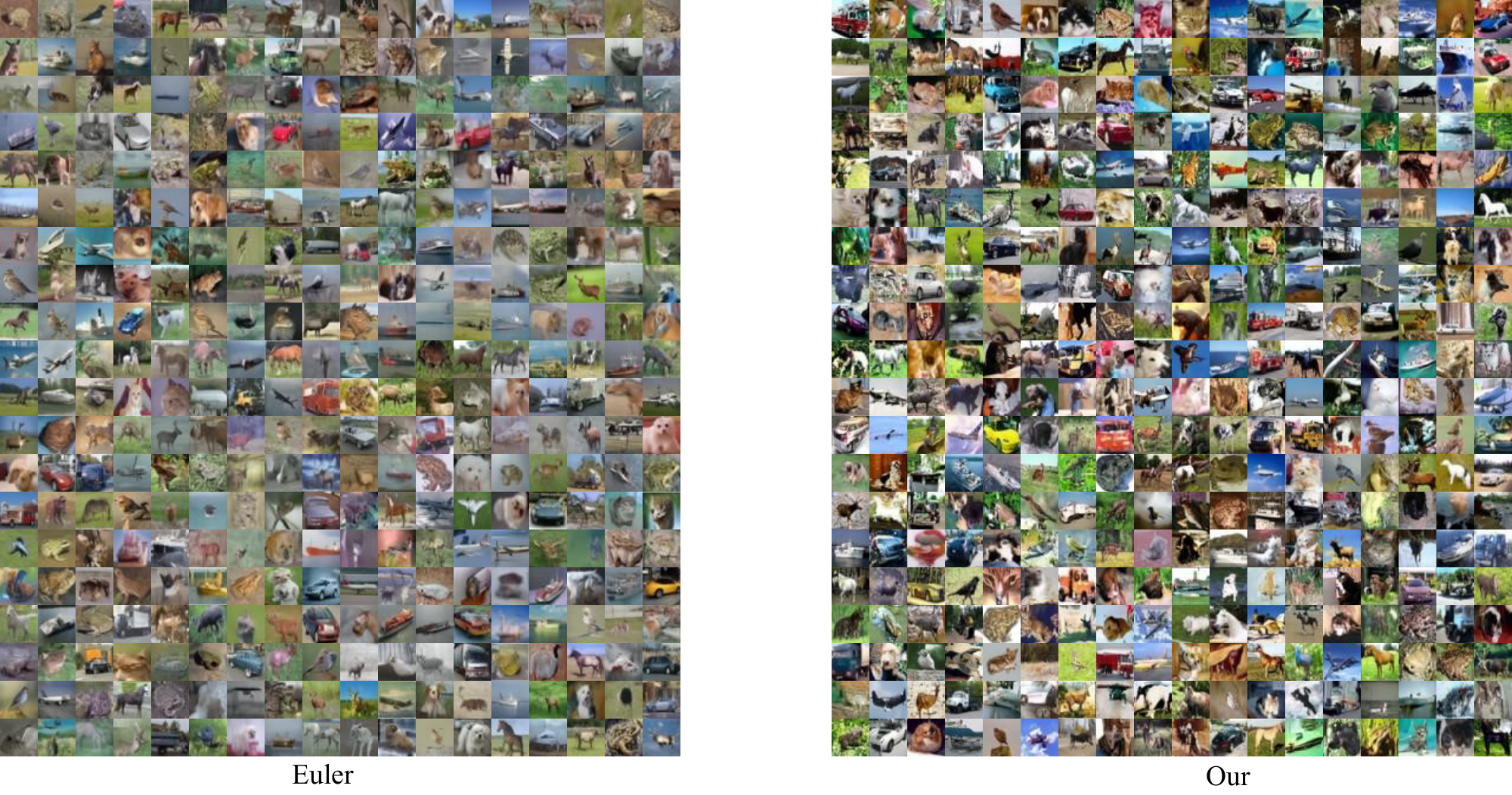}
    \caption{Qualitative examples for the unconditional generation on CIFAR-10 under $\text{NEF}=10$.}
    \label{fig:m_cifar10}
\end{figure*}
\begin{table}
    \centering
    \caption{Ablation study for different NEFs among different solvers on CIFAR-10. We compared the FID $\downarrow$ among different baselines with different time steps $T$.}
    \begin{tabular}{ccccc}
    \hline \multirow{2}{*}{ Sampling Method } & \multicolumn{4}{c}{ NFE } \\
    \cline { 2 - 5 } & 7 & 8 & 9 & 10 \\
    \hline Euler solver& 23.43 &  19.23  & 16.07  & 13.79 \\
    Flow-solver (our)  & \textbf{8.93} &  \textbf{7.70} & \textbf{7.27} & \textbf{6.62}\\
    \midrule
    \multirow{2}{*}{ Heun solver} & 14 & 16 &18 &20\\
     \cline { 2 - 5 }& 9.33 & 8.35 & 7.63 & 7.05 \\
    \hline
    \end{tabular}
    \label{tab:ala_NEF}
\end{table}
\begin{figure*}
    \centering
    \includegraphics[width=\textwidth]{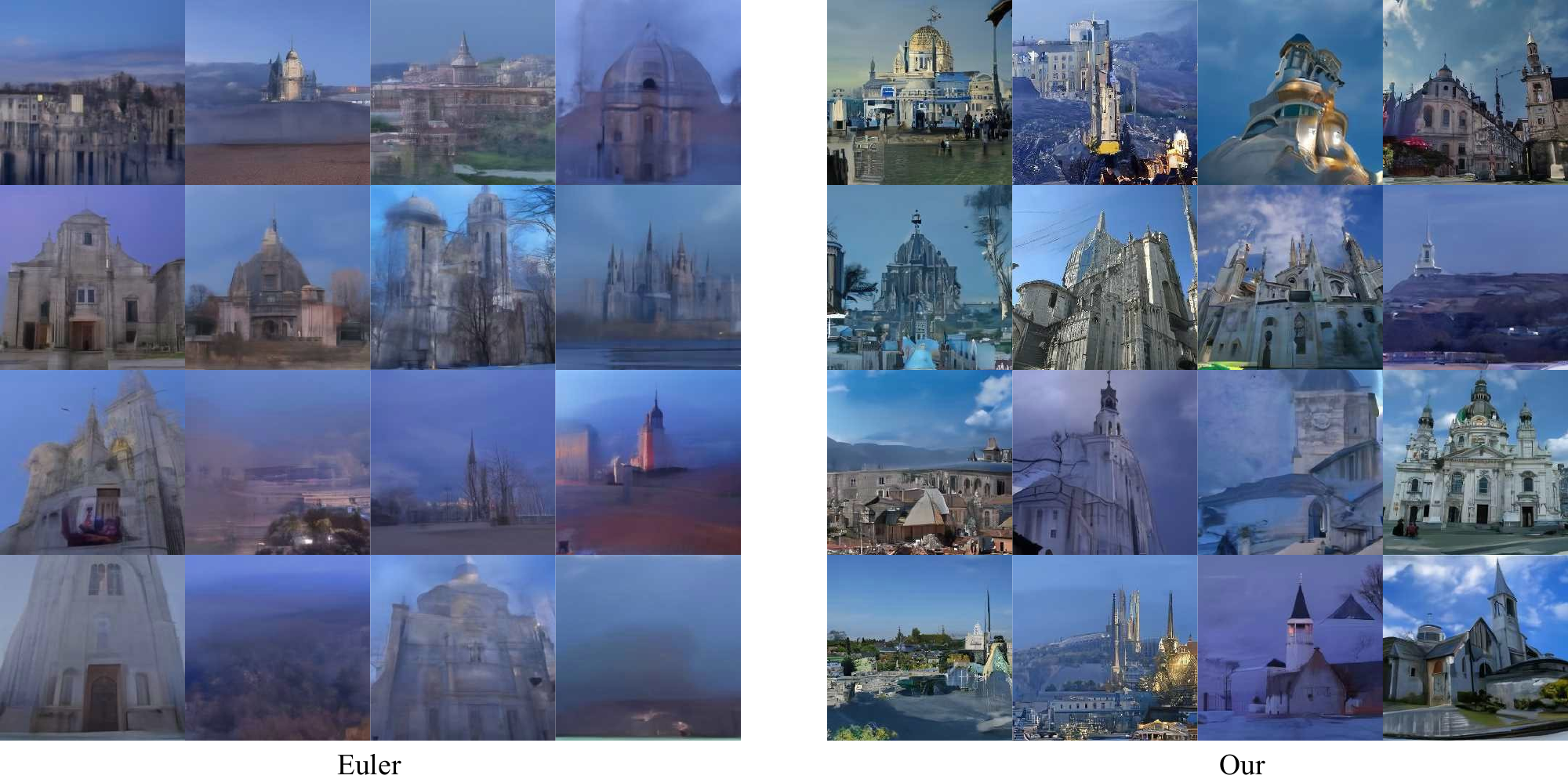}
    \caption{Qualitative examples for the unconditional generation on LSUN-Church under $\text{NEF}=8$.}
    \label{fig:m_church_8}
\end{figure*}
\begin{figure*}
    \centering
    \includegraphics[width=\textwidth]{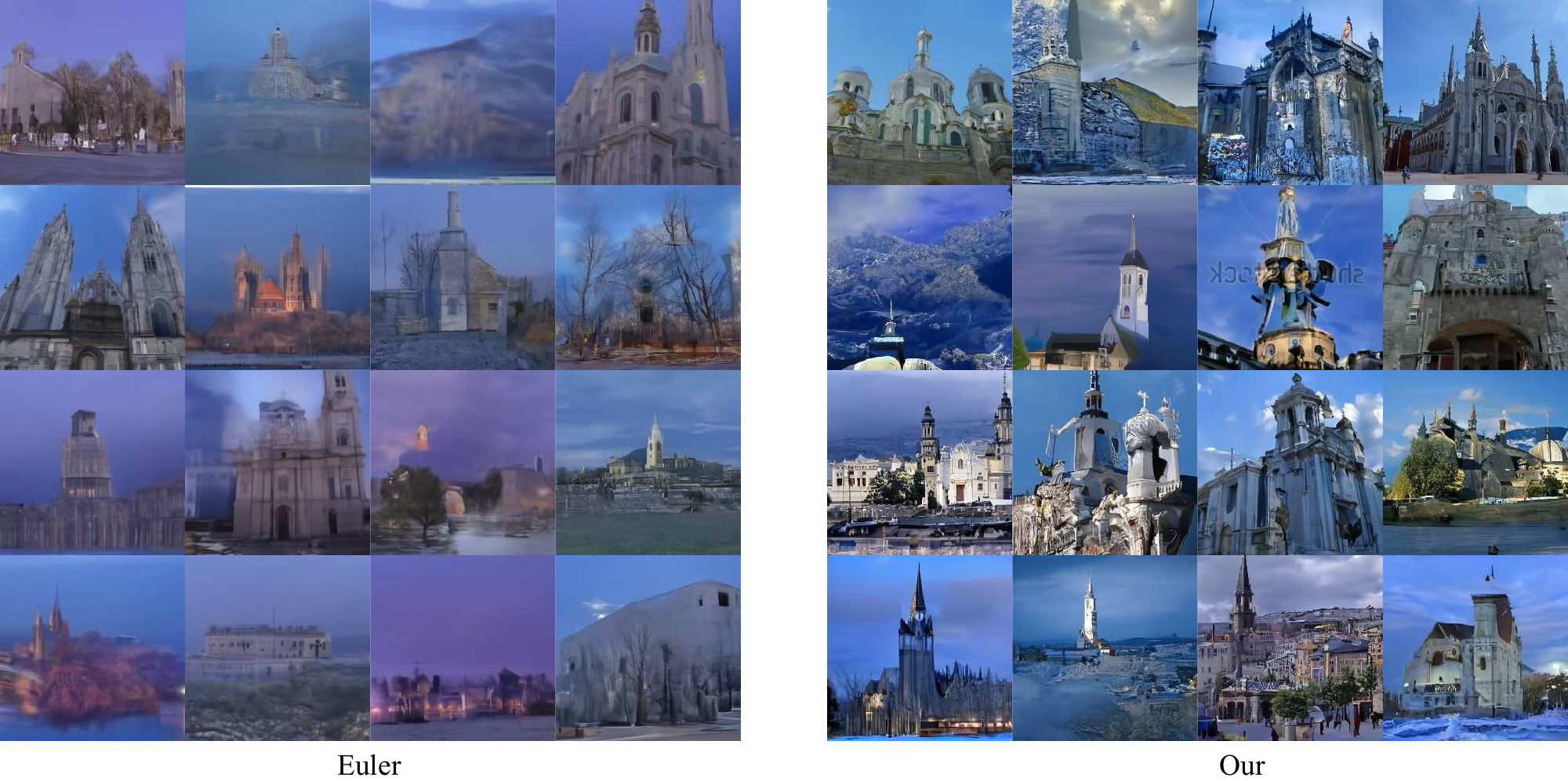}
    \caption{Qualitative examples for the unconditional generation on LSUN-Church under $\text{NEF}=9$.}
    \label{fig:m_church_9}
\end{figure*}
\begin{figure*}
    \centering
    \includegraphics[width=\textwidth]{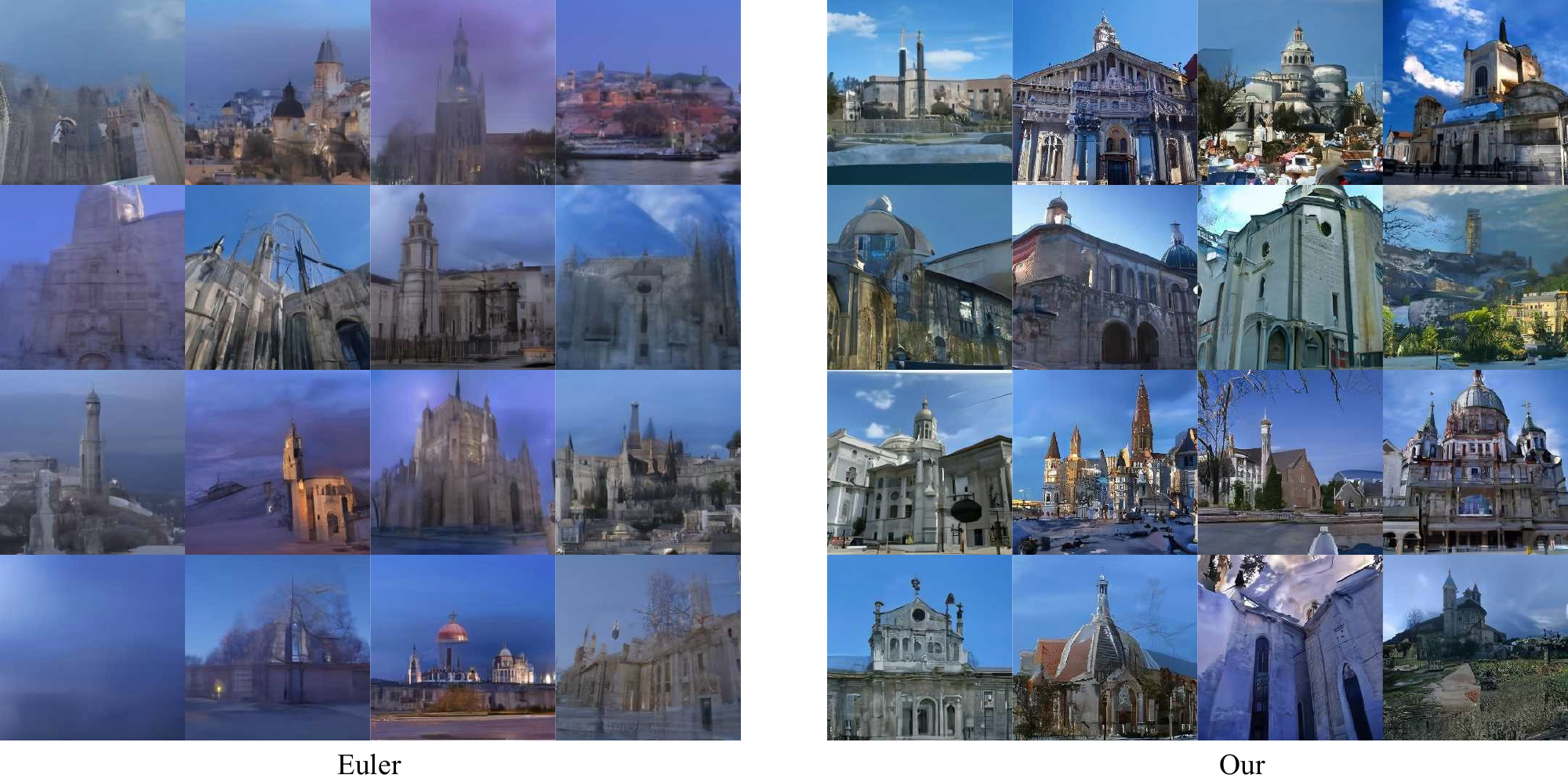}
    \caption{Qualitative examples for the unconditional generation on LSUN-Church under $\text{NEF}=10$.}
    \label{fig:m_church_10}
\end{figure*}
\begin{figure*}
    \centering
    \includegraphics[width=\textwidth]{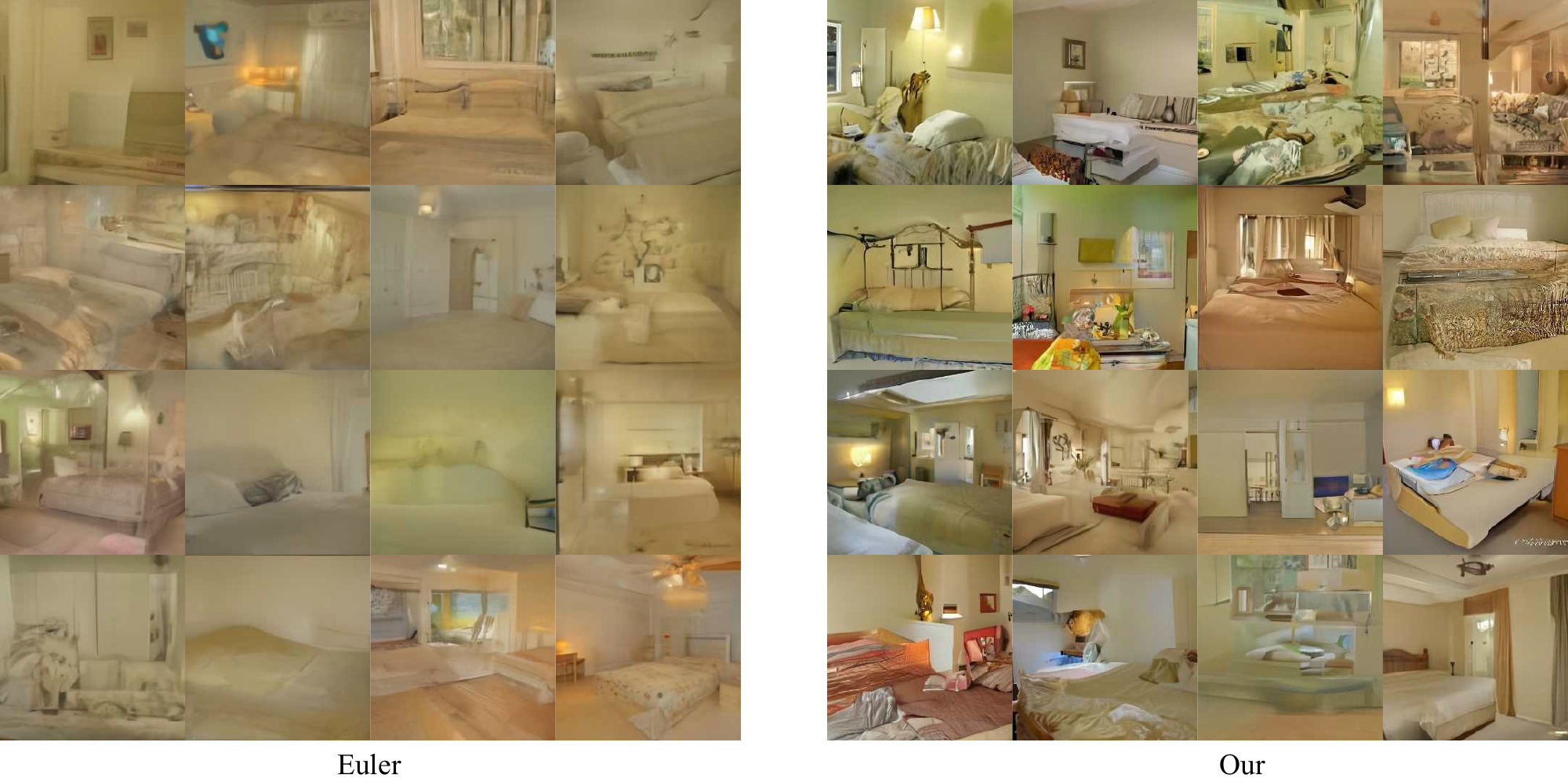}
    \caption{Qualitative examples for the unconditional generation on LSUN-Bedroom under $\text{NEF}=7$.}
    \label{fig:m_bed_7}
\end{figure*}
\begin{figure*}
    \centering
    \includegraphics[width=\textwidth]{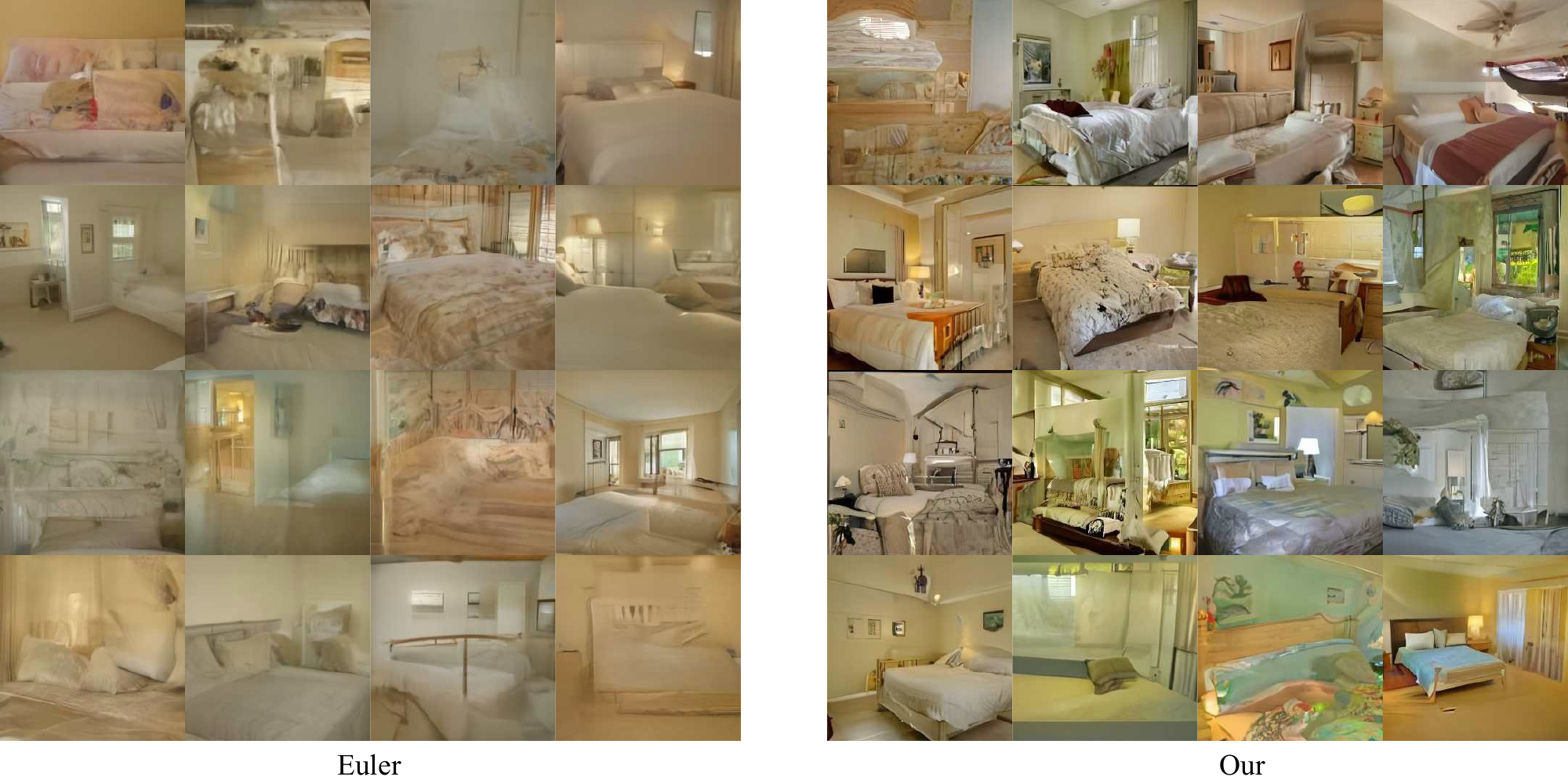}
    \caption{Qualitative examples for the unconditional generation on LSUN-Bedroom under $\text{NEF}=8$.}
    \label{fig:m_bed_8}
\end{figure*}
\begin{figure*}
    \centering
    \includegraphics[width=\textwidth]{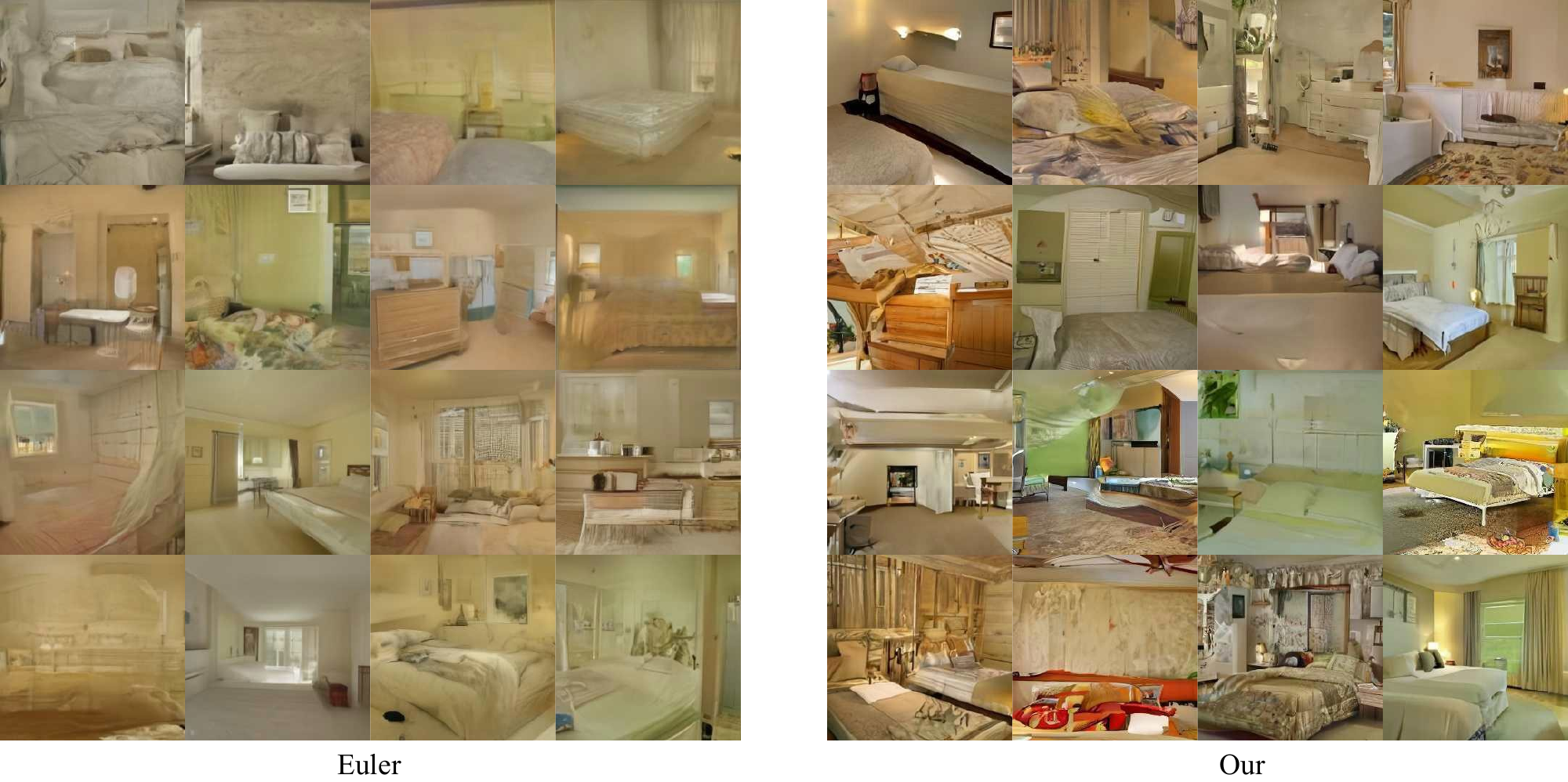}
    \caption{Qualitative examples for the unconditional generation on LSUN-Bedroom under $\text{NEF}=9$.}
    \label{fig:m_bed_9}
\end{figure*}
\begin{figure*}
    \centering
    \includegraphics[width=\textwidth]{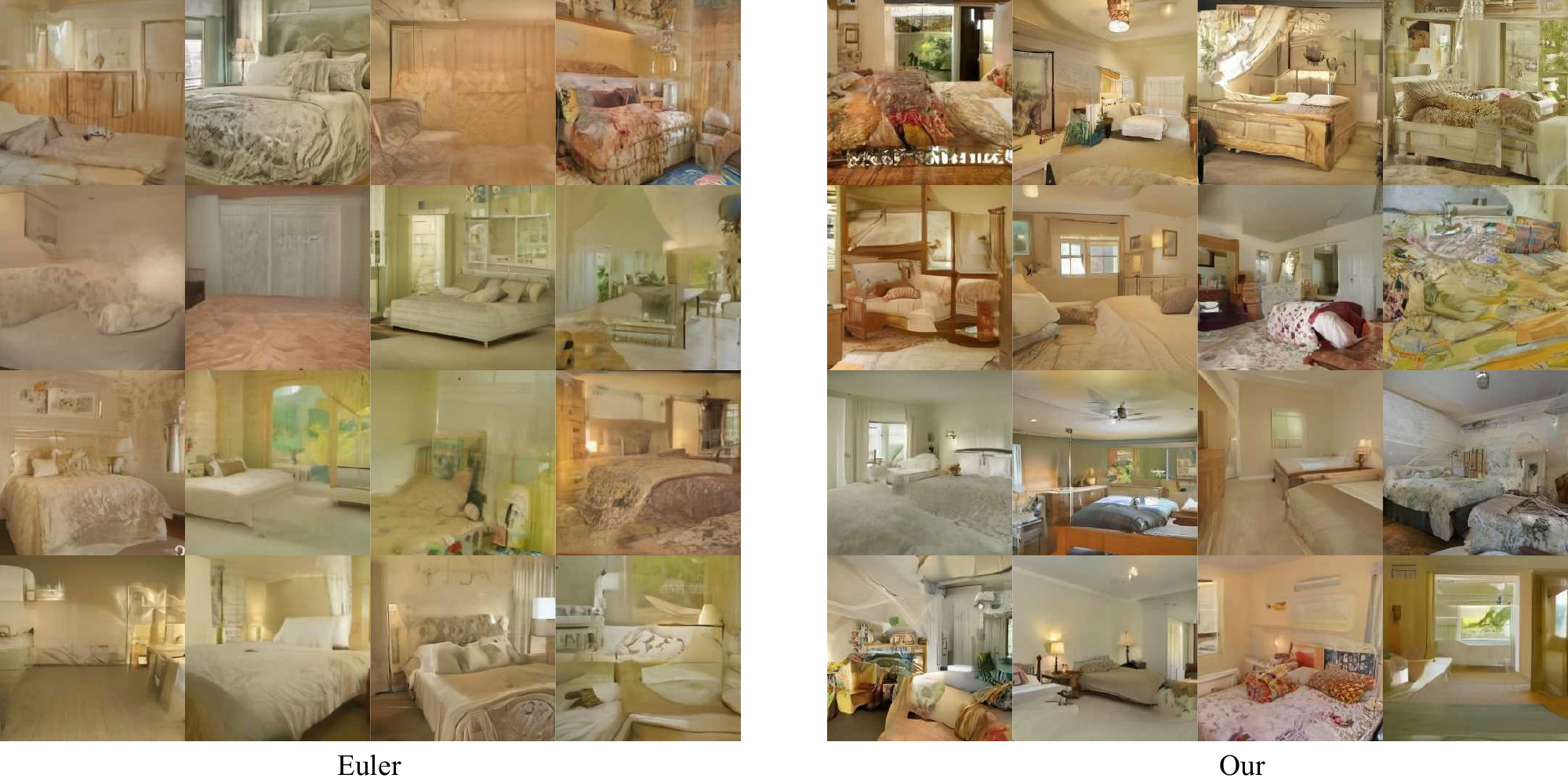}
    \caption{Qualitative examples for the unconditional generation on LSUN-Bedroom under $\text{NEF}=10$.}
    \label{fig:m_bed_10}
\end{figure*}
\begin{figure*}
    \centering
    \includegraphics[width=\textwidth]{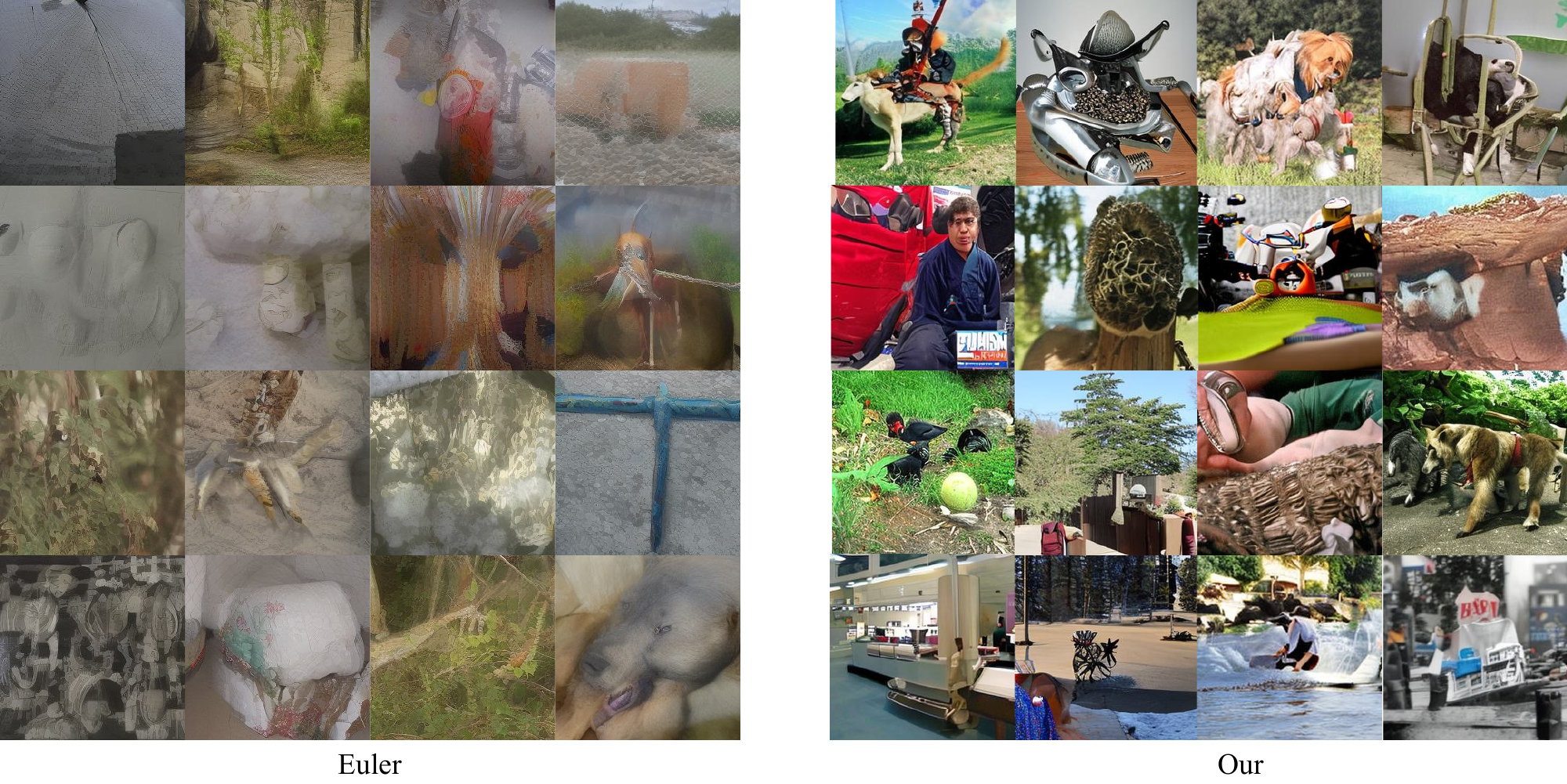}
    \caption{Qualitative examples for the unconditional generation on ImageNet under $\text{NEF}=7$.}
    \label{fig:m_imagenet_7}
\end{figure*}
\begin{figure*}
    \centering
    \includegraphics[width=\textwidth]{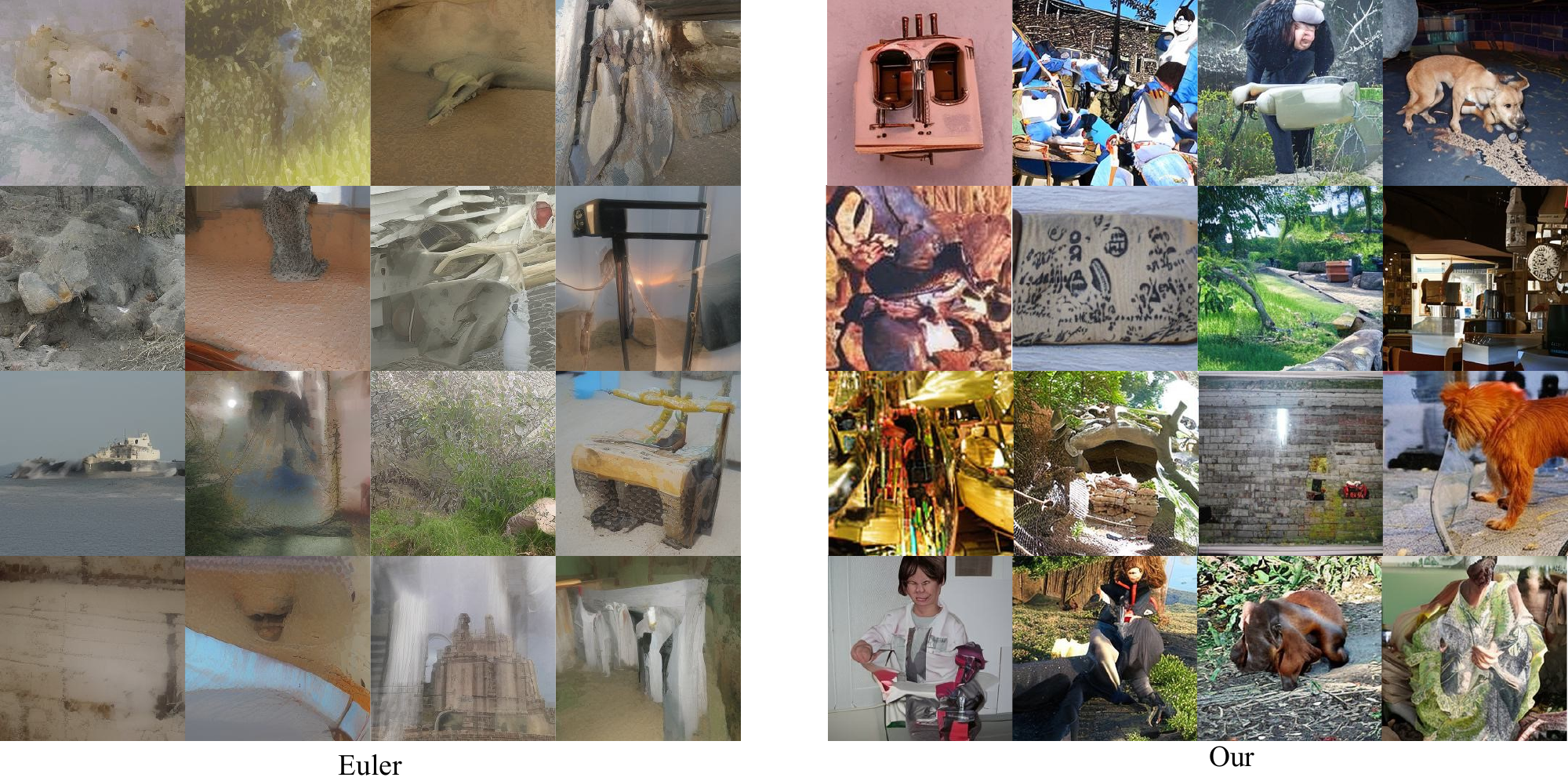}
    \caption{Qualitative examples for the unconditional generation on ImageNet under $\text{NEF}=8$.}
    \label{fig:m_imagenet_8}
\end{figure*}
\begin{figure*}
    \centering
    \includegraphics[width=\textwidth]{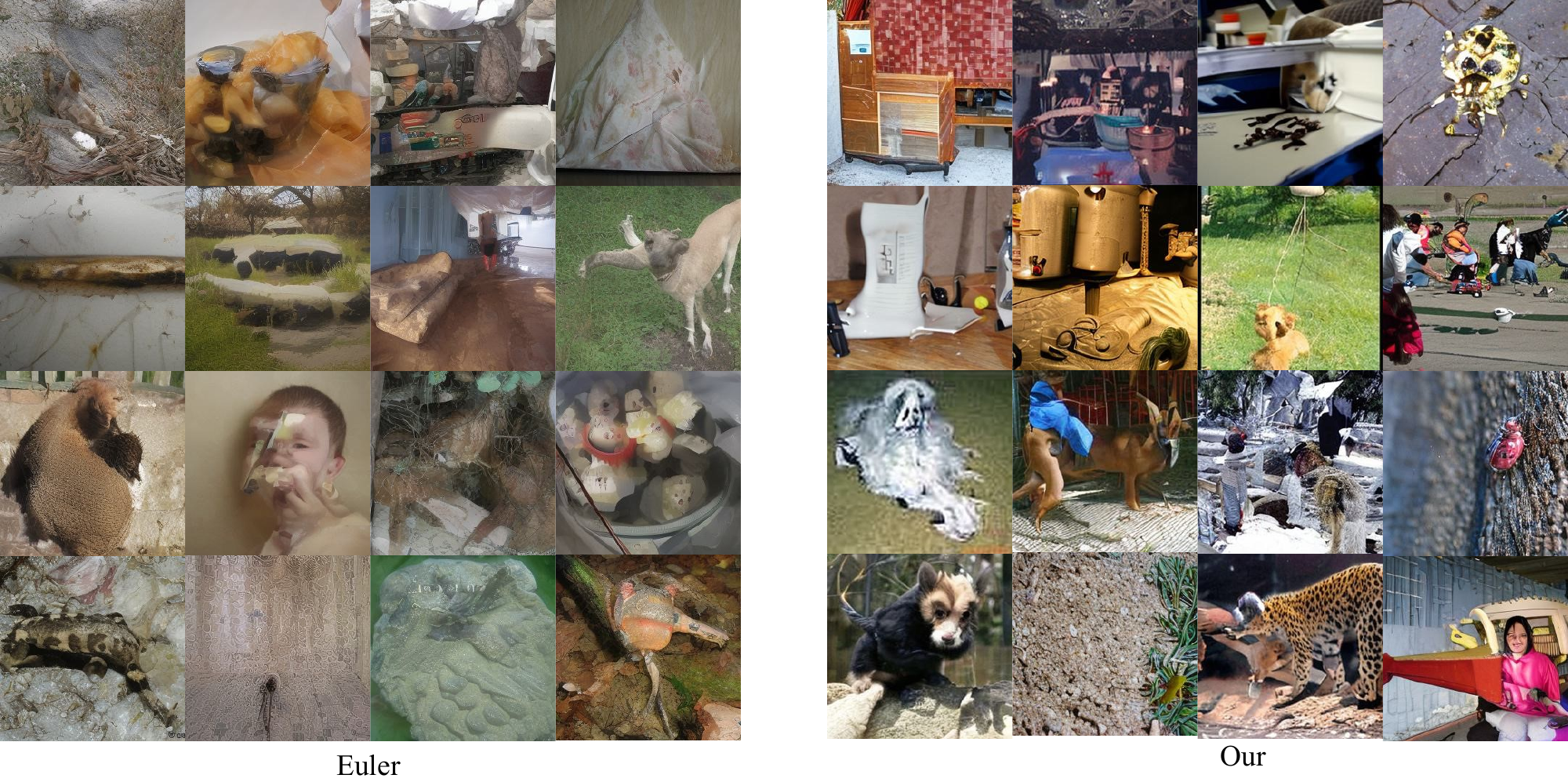}
    \caption{Qualitative examples for the unconditional generation on ImageNet under $\text{NEF}=9$.}
    \label{fig:m_imagenet_9}
\end{figure*}
\begin{figure*}
    \centering
    \includegraphics[width=\textwidth]{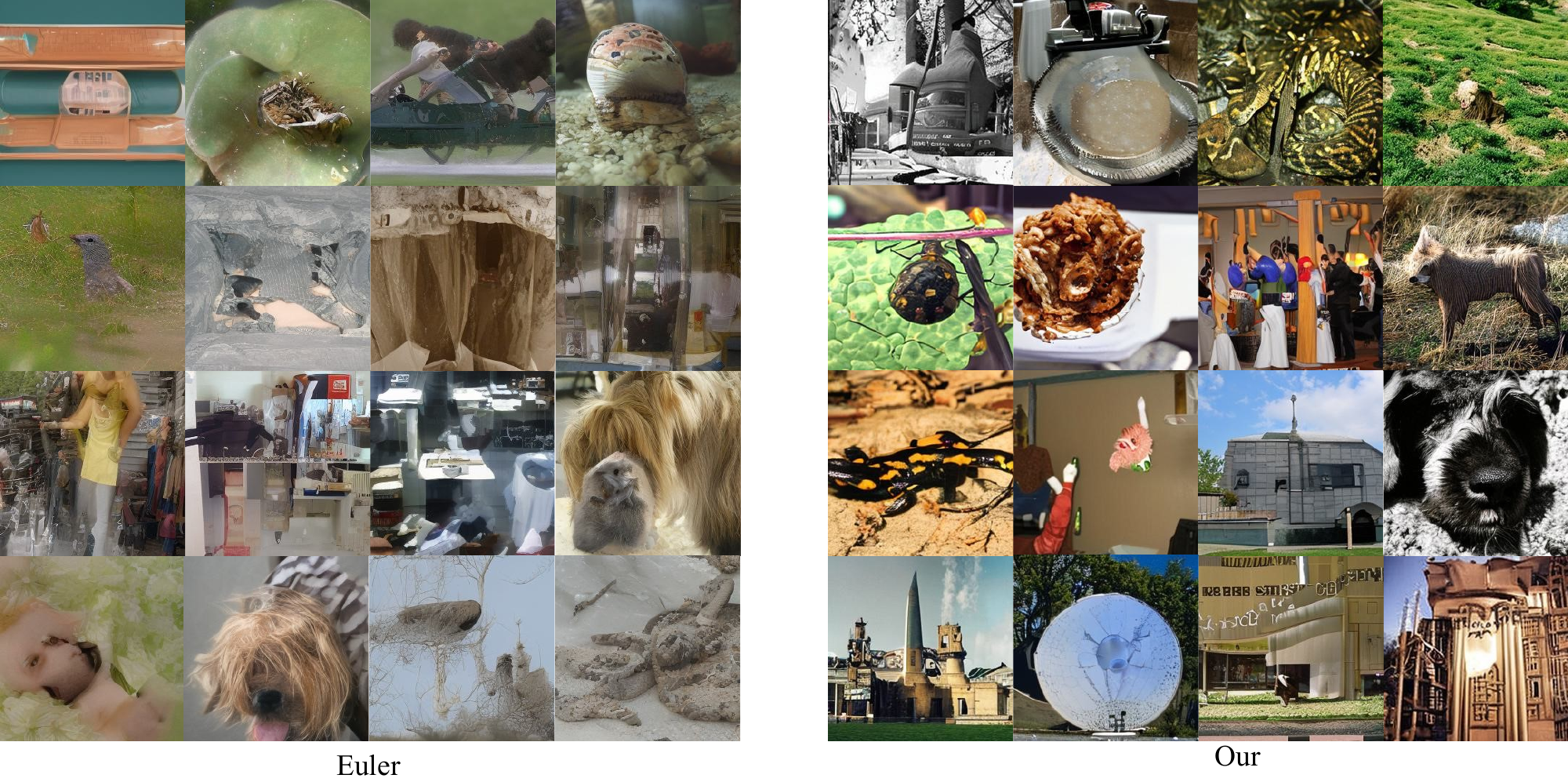}
    \caption{Qualitative examples for the unconditional generation on ImageNet under $\text{NEF}=10$.}
    \label{fig:m_imagenet_10}
\end{figure*}
\begin{figure*}
    \centering
    \includegraphics[width=\textwidth]{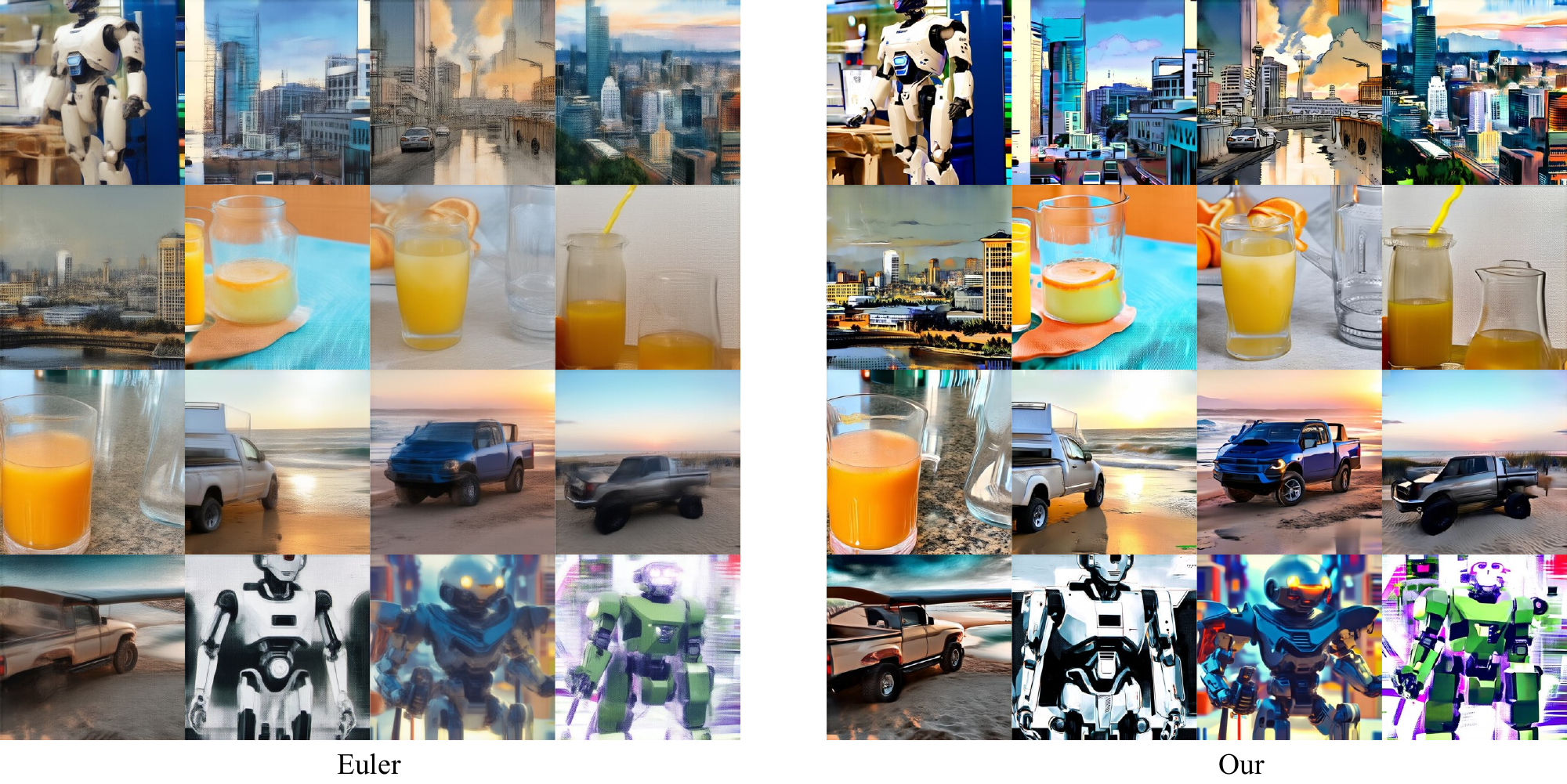}
    \caption{Qualitative examples for the conditional generation based on Stable Diffusion 3.0 under $\text{NEF}=6$ and $\text{CFG}=2$.}
    \label{fig:m_sd_2_6}
\end{figure*}
\begin{figure*}
    \centering
    \includegraphics[width=\textwidth]{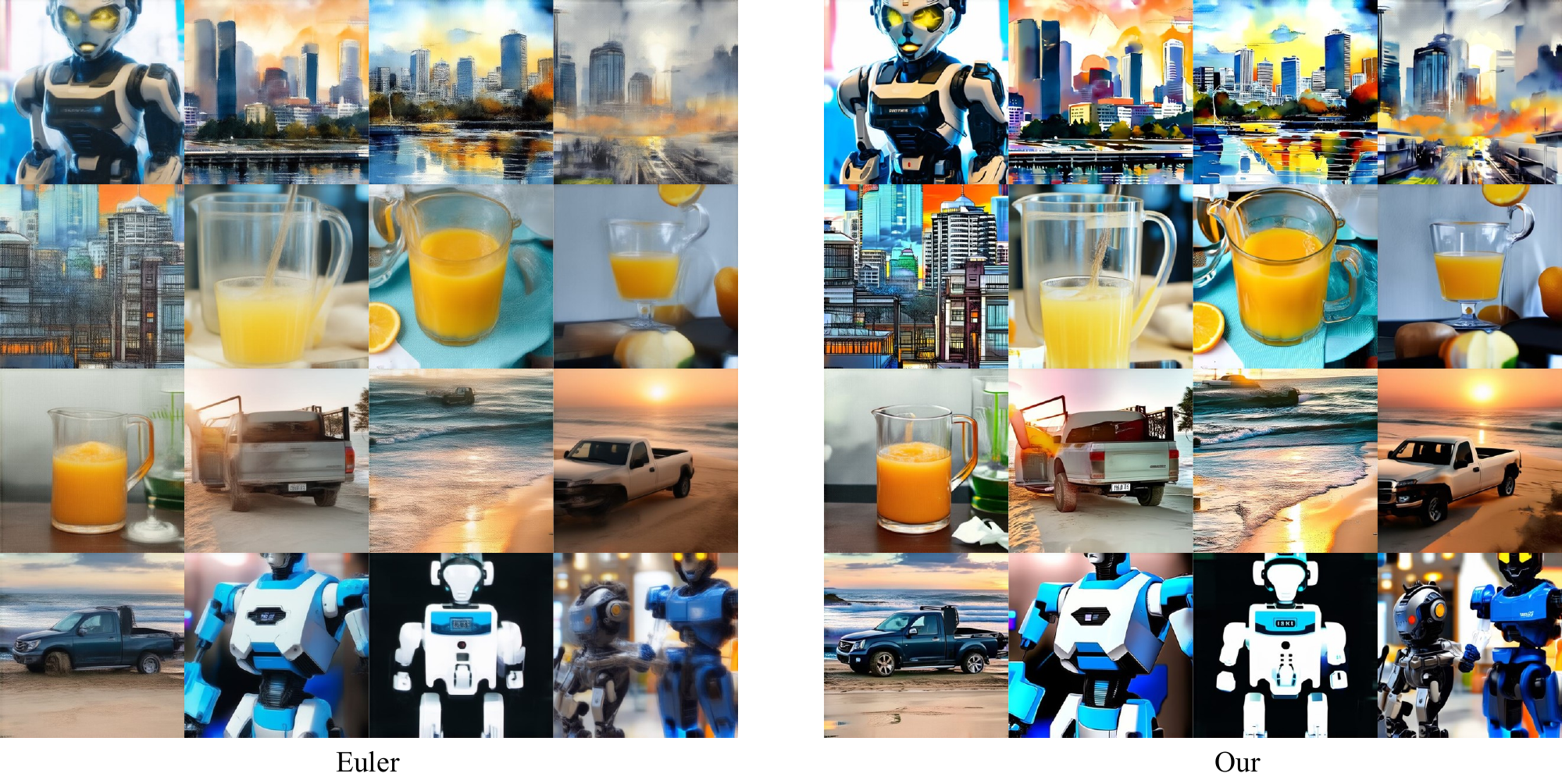}
    \caption{Qualitative examples for the conditional generation based on Stable Diffusion 3.0 under $\text{NEF}=7$ and $\text{CFG}=2$.}
    \label{fig:m_sd_2_7}
\end{figure*}
\begin{figure*}
    \centering
    \includegraphics[width=\textwidth]{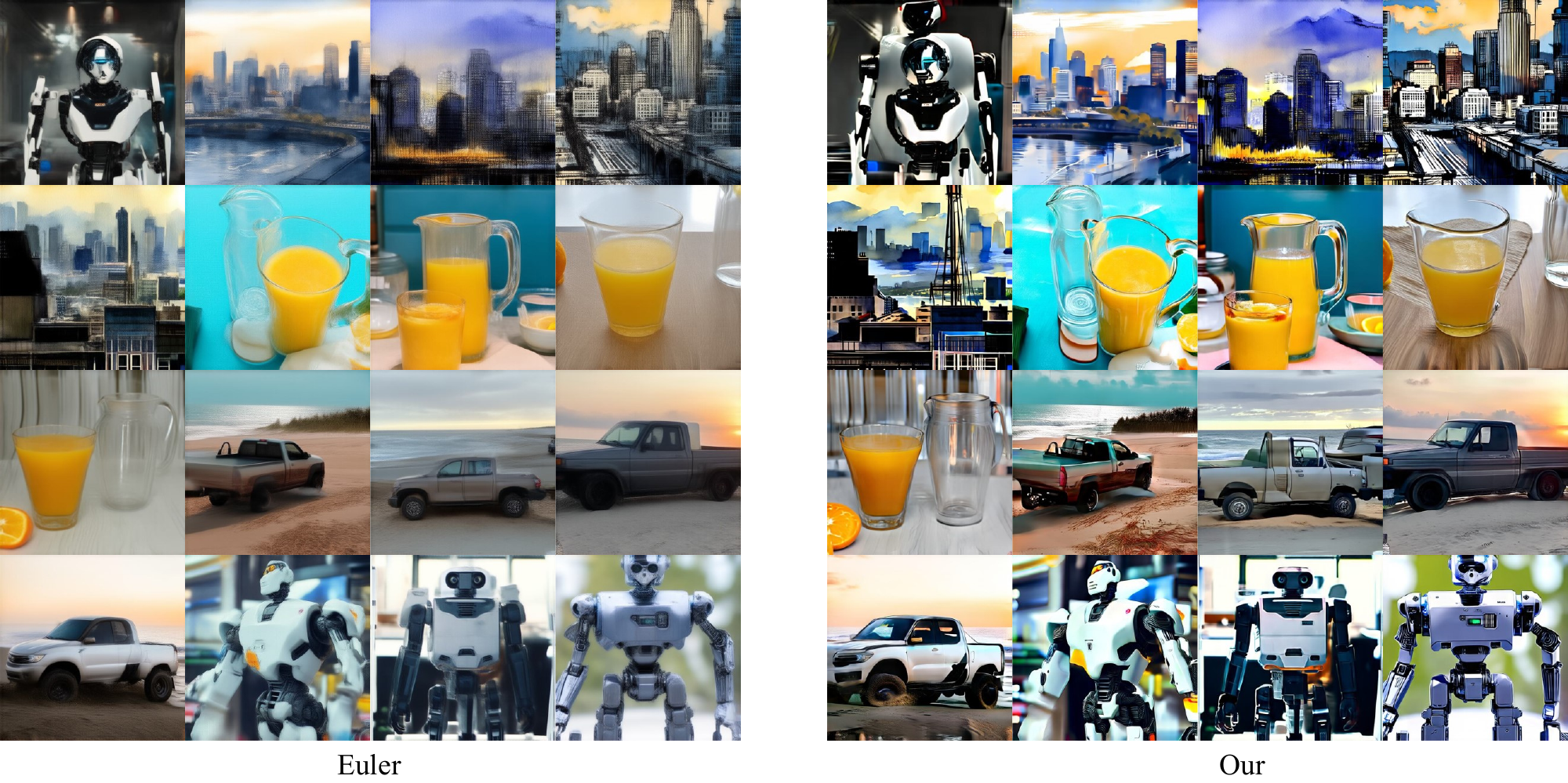}
    \caption{Qualitative examples for the conditional generation based on Stable Diffusion 3.0 under $\text{NEF}=6$ and $\text{CFG}=3$.}
    \label{fig:m_sd_3_6}
\end{figure*}
\begin{figure*}
    \centering
    \includegraphics[width=\textwidth]{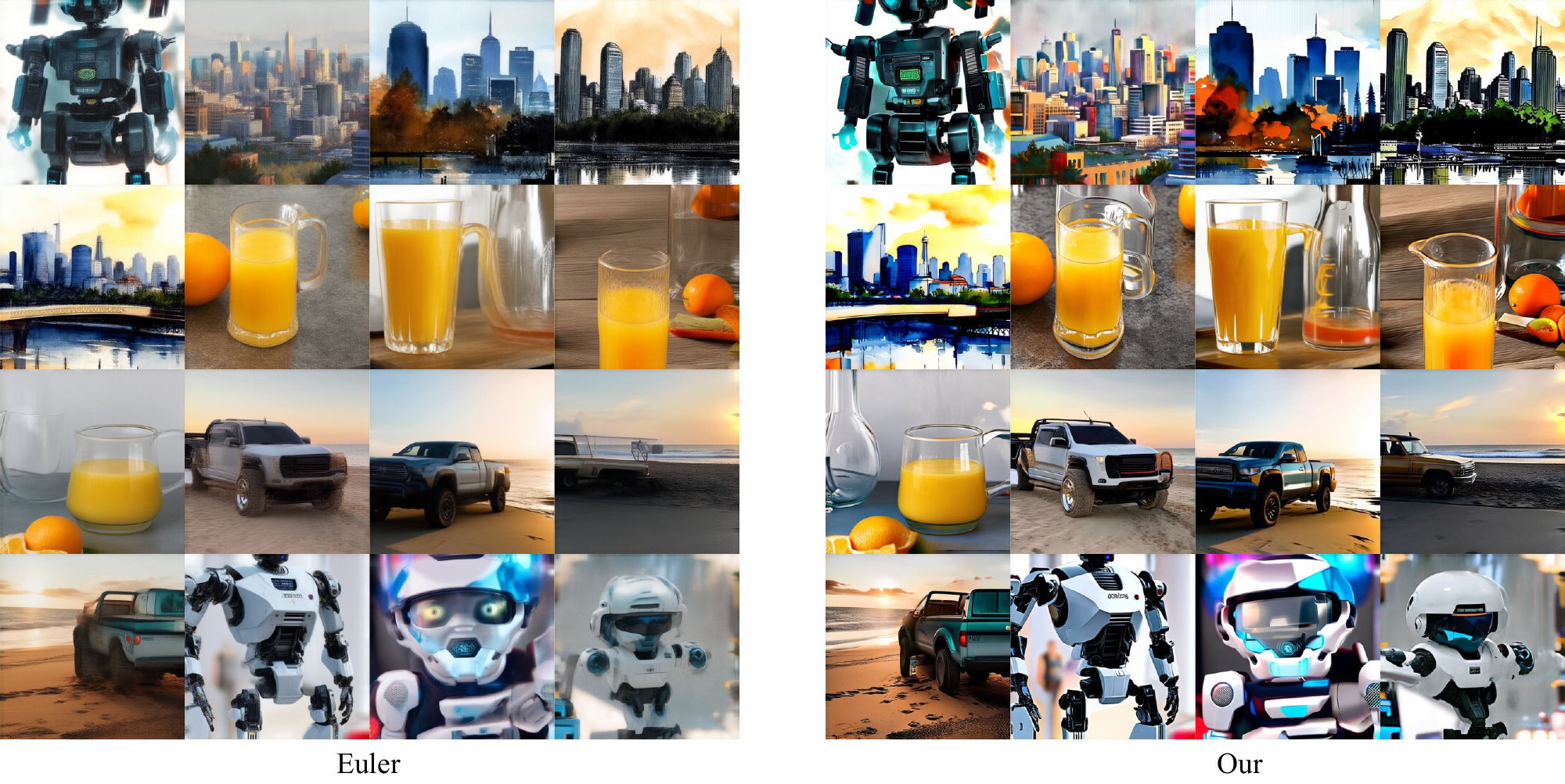}
    \caption{Qualitative examples for the conditional generation based on Stable Diffusion 3.0 under $\text{NEF}=7$ and $\text{CFG}=3$.}
    \label{fig:m_sd_3_7}
\end{figure*}
\begin{figure*}
    \centering
    \includegraphics[width=\textwidth]{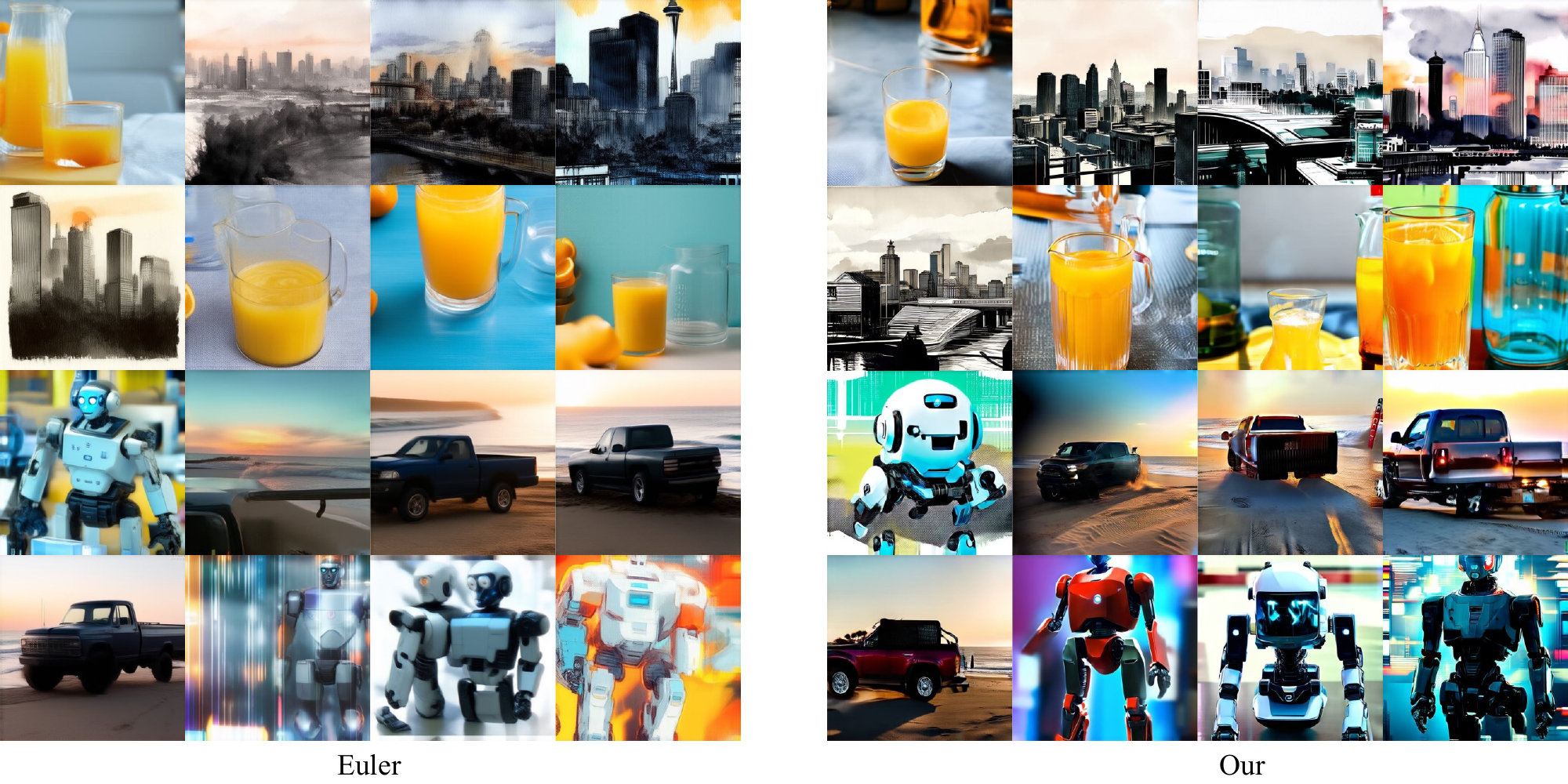}
    \caption{Qualitative examples for the conditional generation based on Stable Diffusion 3.0 under $\text{NEF}=6$ and $\text{CFG}=4$.}
    \label{fig:m_sd_4_6}
\end{figure*}
\begin{figure*}
    \centering
    \includegraphics[width=\textwidth]{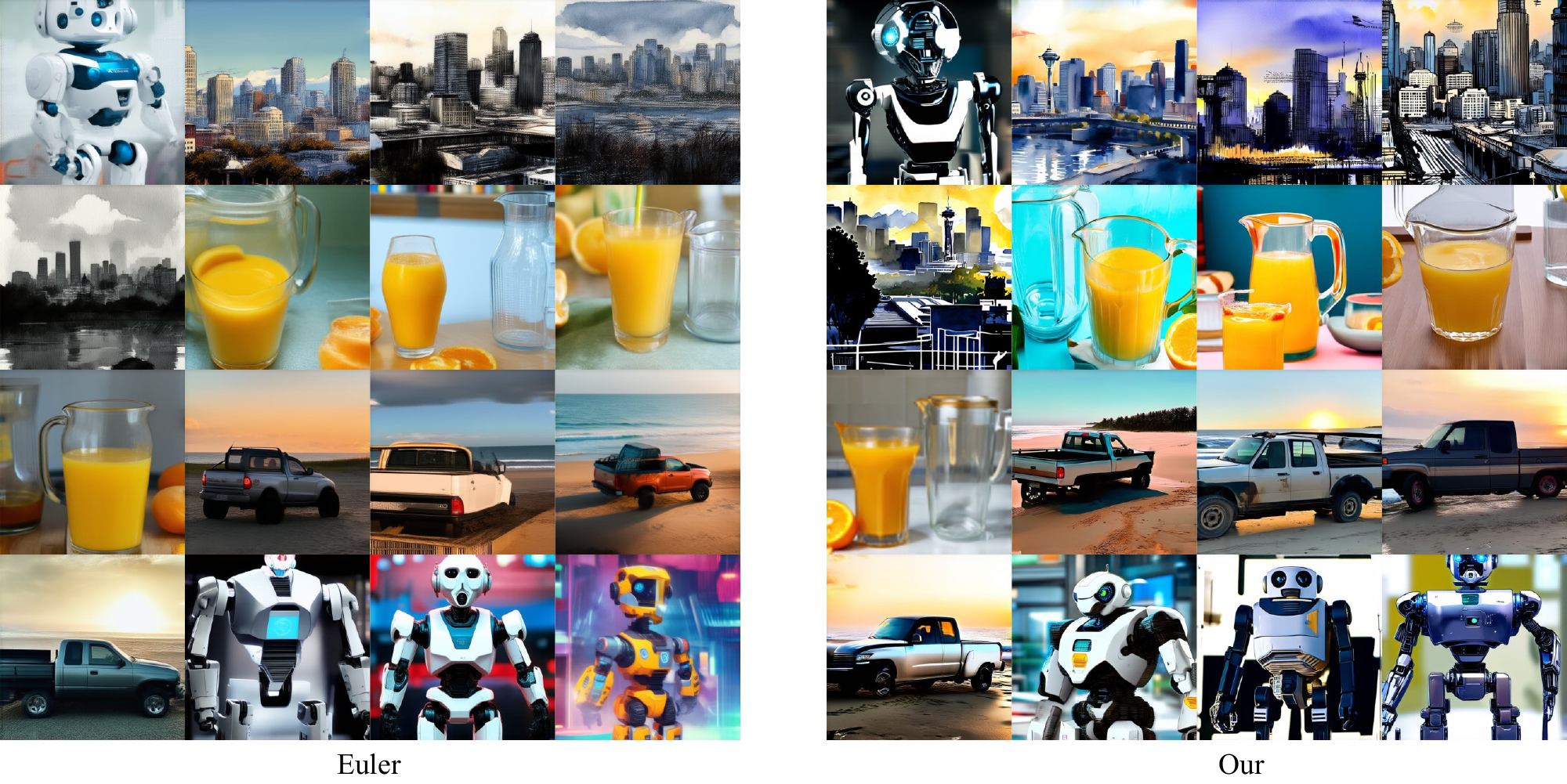}
    \caption{Qualitative examples for the conditional generation based on Stable Diffusion 3.0 under $\text{NEF}=7$ and $\text{CFG}=4$.}
    \label{fig:m_sd_4_7}
\end{figure*}

\end{document}